\documentclass[11pt,letterpaper]{article} % For LaTeX2e

\usepackage[OT1]{fontenc} 
\usepackage[colorlinks]{hyperref}
\usepackage{url}
\usepackage[numbers]{natbib}

\usepackage{fullpage,times}
\usepackage{amsthm,amsmath,amsfonts} 
\usepackage{epsf,epsfig,caption,subcaption,graphicx} 
\usepackage[ruled,boxed]{algorithm2e}
\usepackage{tikz}

\newtheorem{theorem}{Theorem}[section]
\newtheorem{lemma}[theorem]{Lemma} 
\newtheorem{proposition}[theorem]{Proposition} 

\theoremstyle{definition}

\newtheorem{assumption}[theorem]{Assumption}

\theoremstyle{remark}

\def\var{\mbox{Var}} 

\def\P{\mathbb{P}} 
\def\E{\mathbb{E}} 

\def\exp{\mbox{exp}}
\def\pa{\mbox{Pa}}

\def\de{\mbox{De}}
\def\nd{\mbox{Nd}}

\def\S{\mathcal{S}}
\def\hatS{\widehat{\mathcal{S}}}

\DeclareMathAlphabet\mathbfcal{OMS}{cmsy}{b}{n}

\begin{document}

\begin{center}
	{\bf{\LARGE{High-Dimensional Poisson Structural Equation Model Learning via $\ell_1$-Regularized Regression}}}
	%Directed Acyclic Graphical Models via Exponential Family Distributions
	\vspace*{.1in}
	
	\begin{tabular}{cc}
		Gunwoong Park$^1$, \quad Sion Park$^1$\\
	\end{tabular}
	
	\vspace*{.1in}
	
	\begin{tabular}{c}
		$^1$ Department of Statistics, University of Seoul \\
	\end{tabular}
	
	\vspace*{.1in}
	%\today
	
\end{center}

\begin{abstract}%   <- trailing '%' for backward compatibility of .sty file
In this paper, we develop a new approach to learning high-dimensional Poisson structural equation models from only observational data without strong assumptions such as faithfulness and a sparse moralized graph. A key component of our method is to decouple the ordering estimation or parent search where the problems can be efficiently addressed using $\ell_1$-regularized regression and the moments relation. We show that sample size $n = \Omega( d^{2} \log^{9} p)$ is sufficient for our polynomial time Moments Ratio Scoring (MRS) algorithm to recover the true directed graph, where $p$ is the number of nodes and $d$ is the maximum indegree. We verify through simulations that our algorithm is statistically consistent in the high-dimensional $p>n$ setting, and performs well compared to state-of-the-art ODS, GES, and MMHC algorithms. We also demonstrate through multivariate real count data that our MRS algorithm is well-suited to estimating DAG models for multivariate count data in comparison to other methods used for discrete data.
\end{abstract}

\section{Introduction}

% Importance of DAG models: causal relationships % Introduction: identifiability of DAG or MEC
Directed acyclic graphical (DAG) models, also referred to as Bayesian networks, are popular probabilistic statistical models to analyze and visualize (functional) causal or directional dependence relationships among random variables.(see e.g.,~\citealp{kephart1991directed,friedman2000using, doya2007bayesian, peters2014identifiability}). 
However, learning DAG models from only observational data is a notoriously difficult problem due to non-identifiability and exponentially growing computational complexity. Prior works have addressed the question of identifiability for different classes of joint distribution $\P(G)$. \cite{frydenberg1990chain} and \cite{heckerman1995learning} show the Markov equivalence class (MEC) where graphs that belong to the same MEC have the same conditional independence relations. \cite{spirtes2000causation}, \cite{chickering2003optimal}, \cite{tsamardinos2003towards} and \cite{zhang2016three} show that the underlying graph of a DAG model is recoverable up to the MEC under faithfulness or related assumptions that can be very restrictive~\citep{uhler2013geometry}.
%Hence most of works has been focused on learning MEC or the skeleton of a DAG by leaving some edges undirected. 

% Computational Difficulties of learning DAG
Also well studied is how learning a DAG model is computationally non-trivial due to the super-exponent-\\ially growing size of the space of DAGs in the number of nodes. Hence, it is NP-hard to search DAG space~\citep{chickering1994learning, chickering1996learning}, and many existing algorithms such as PC~\citep{spirtes2000causation}, Greedy Equivalence Search (GES)~\citep{chickering2003optimal}, Min-Max Hill Climbing (MMHC)~\citep{tsamardinos2006max} and Greedy DAG Search (GDS)~\citep{peters2014identifiability}, take greedy search methods that may not guarantee  to recover the true MEC.

% Update: include references %
% Identifiability: Motivation for learning high dimensional graphs:
Recently, a number of fully identifiable classes of DAG models have been introduced \citep{shimizu2006linear,hoyer2009nonlinear, peters2012identifiability,peters2014identifiability, park2015learning, park2017learning, ghoshal2018learning, park2019identifiability}. In addition, some of these models can be successfully learned from high-dimensional data by decomposing the DAG learning problem into ordering estimation and skeleton estimation~\citep{shimizu2011directlingam, buhlmann2014cam, ghoshal2017learning, drton2018causal}. The main reasoning is that if ordering is known or recoverable, learning a directed graphical model is as hard as learning an undirected graphical model or Markov random field (MRF). \cite{meinshausen2006high}, \cite{wainwright2006high}, \cite{ravikumar2011high} and \cite{yang2015graphical} show that sparse undirected graphs can be estimated via $\ell_1$-regularized regression in high-dimensional settings under suitable conditions.

% Our Focus % Problem of Poisson DAG % Why Counrt Data %
In this paper, we focus on learning Poisson DAG models~\citep{park2015learning, park2017learning} for multivariate count data in high-dimensional settings since large-scale multivariate \emph{count data} frequently arises in many fields, such as high-throughput genomic sequencing data, spatial incidence data, sports science data, and disease incidence data. Like learning the Poisson undirected graphical model or MRF introduced in~\cite{yang2015graphical}, where the sample bound is $\Omega(d_{m}^2 \log^3 p)$, it is not surprising that Poisson DAG models can be learned in high dimensional settings when the indegree of the graph $d$ is bounded. \cite{park2017learning} establishes the consistency of learning Poisson DAG models with the sample bound $n = \Omega(\max\{ d_{m}^4 \log^{12} p, \log^{5+d} p\} )$ where  $d_m$ is the maximum degree of the moralized graph and $d$ is the maximum indegree of a graph. This huge sample complexity difference between directed and undirected graphical models is induced mainly for three reasons: (i) nonexistence ordering, (ii) the known parametric functional form (the standard log link) for the dependencies, and (iii) the restrictive non-positive parameter space in Poisson MRFs (see details in \citealp{yang2015graphical}). 

% Objectives Algorithm
The main objective of this work is to propose a new milder identifiability assumption for Poisson DAG models, and to develop a new polynomial time approach, called Moments Ratio Scoring (MRS), for learning a high-dimensional Poisson structural equation models (SEM), that is a Poisson DAG model where the parametric functional form for the dependencies is known while the parameters are unbounded and unknown. We address the question of learning high-dimensional Poisson SEMs under the causal sufficiency assumption that all relevant variables have been observed. However, we do not require the sparse moralized graph and faithfulness assumption that might be restrictive~\citep{uhler2013geometry}. 

%%% Algorithm %%%
The MRS algorithm combines the idea of the mean-variance (moments) relation for recovering an ordering, and the sparsity-encouraging $\ell_1$-regularized regression in finding the parents of each node. We provide its sufficient conditions and sample complexity $n = \Omega( d^{2} \log^{9} p)$ under which the MRS algorithm recovers the Poisson SEM with a high probability in the high-dimensional $p > n$ setting. The sample complexity of $n = \Omega( d^{2} \log^{9} p)$ is close to the information-theoretic limit of $\Omega( d \log p )$ for learning sparse DAG models with any exponential family distributions~\citep{ghoshal2017information}. We point out that the sample complexity does not depend on the maximum degree of the moralized graph, $d_m$, but on the indegree of a DAG, $d$. Since a sparse directed graph does not necessarily lead to the sparse moralized graph (e.g., a star graph in Fig.~\ref{fig:star}), to the best of our knowledge, the proposed algorithm is the most efficient and probable for learning sparse Poisson SEMs.
%%%%
We demonstrate through simulations and a real baseball data application involving multivariate count data that our MRS algorithm performs better than state-of-the-art OverDispersion Scoring (ODS)~\citep{park2015learning}, GES~\citep{chickering2003optimal}, MMHC~\citep{tsamardinos2006max}, and Poisson MRF learning (PMRF) algorithms~\citep{yang2015graphical}, on average, in terms of the both run-time and accuracy of recovering a graph structure and its MEC. In our simulation study, we consider both the extremely sparse ($d=1$) and sparse ($d=10$) high-dimensional settings. Our real data example involving MLB player statistics for 2003 season shows that our MRS algorithm is applicable to multivariate count data while the PMRF algorithm finds too many edges, and the MMHC algorithm tends to select very few edges when variables represent counts. We also investigate the accuracy of our MRS algorithm when samples are generated from general Poisson DAG models and (truncated) Poisson MRFs. The simulation results empirically verify that the MRS algorithm can consistently recover the true edges. 

\subsection{Our Contributions}

We summarize the major contributions of the paper as follows:
\begin{itemize}
	\item We introduce a milder identifiability condition for Poisson DAG models for multivariate count data. 
	
	\item We develop the reliable and scalable lasso-based MRS algorithm which learns sparse high-dimensional Poisson SEMs.
	
	\item We provide the more realistic conditions for learning Poisson SEMs in Section~\ref{SecTheo}.
	
	\item We also provide the sample complexity $n = \Omega( d^{2} \log^{9} p)$ under which the MRS algorithm recovers the Poisson SEM. We emphasize that our theoretical result does not depend on the degree of the moralized graph $d_m$, and hence, the MRS algorithm can recover a graph with hub nodes in the high dimensional setting. 
\end{itemize}

To the best of our knowledge, our MRS algorithm is the only provable and realistic method that applies for the high-dimensional multivariate count data when samples are from Poisson SEMs with hub nodes. We must point out that such improved assumptions and sample complexity are not only from our new identifiability condition, but from the additional constraints on the standard log link function for the dependencies. 

%Organization
The remainder of this paper is structured as follows. Section~\ref{SecProb} summarizes the necessary notations and problem settings, Section~\ref{SecPoisson} discusses the Poisson DAG model and its new identifiability condition, and Section~\ref{SecComparison} provides a detailed comparison between Poisson DAG models and MRFs. In Section~\ref{SecAlgorithm}, we introduce our polynomial-time DAG learning algorithm, which we refer to as the Moments Ratio Scoring (MRS). Section~\ref{SecComp} discusses computational complexity of our algorithm, and Section~\ref{SecTheo} provides statistical guarantees for learning Poisson SEMs via the MRS algorithm. Section~\ref{SecNume} empirically evaluates our methods, compared to state-of-the-art ODS, GES, and MMHC algorithms using synthetic data, and confirms that our algorithm is one of the few DAG-learning algorithms that performs well in terms of statistical and computational complexity in low and high-dimensional settings. In addition, we investigate how well the MRS algorithm learns general Poisson DAG models and (truncated) Poisson MRFs using synthetic data. Section~\ref{SecReal} compares our MRS algorithm to the Poisson MRF and MMHC algorithm by analyzing a real 2003 season MLB multivariate count data. Lastly, Section~\ref{SecFuture} discusses some future works. 

\section{Poisson DAG Models}

% Set-sup and notation
% Poisson DAG vs MRF
% Identifiability and ODS algorithm

\label{SecClass}

% Outlines
We first introduce some necessary notations and definitions for DAG models. Then, we give a detailed description of previous work on learning Poisson DAG models~\citep{park2015learning}, and we propose a strictly milder identifiability condition. Lastly, we discuss how Poisson DAG models and MRFs~\citep{yang2015graphical} are related. 

\subsection{Problem Set-up and Notations}

\label{SecProb}

% Graph Definitions 

A DAG $G = (V, E)$ consists of a set of nodes $V = \{1, 2, \cdots, p\}$ and a set of directed edges $E \subset V \times V$ with no directed cycles. A directed edge from node $j$ to $k$ is denoted by $(j,k)$ or $j \rightarrow k$. The set of \emph{parents} of node $k$, denoted by $\pa(k)$, consists of all nodes $j$ such that $(j,k) \in E$. If there is a directed path $j\to \cdots \to k$, then $k$ is called a \emph{descendant} of $j$, and $j$ is an \emph{ancestor} of $k$. The set $\de(k)$ denotes the set of all descendants of node $k$. The \emph{non-descendants} of node $k$ are $\nd(k) := V \setminus (\{k\} \cup \de(k))$. An important property of DAGs is that there exists a (possibly non-unique) \emph{ordering} $\pi = (\pi_1, ...., \pi_p)$ of a directed graph that represents directions of edges such that for every directed edge $(j, k) \in E$, $j$ comes before $k$ in the ordering. Hence, learning a graph is equivalent to learning the ordering and the skeleton that is the set of directed edges without their directions.

%Similar definitions and notations can be found in~\citet{spirtes2000causation, koller2009probabilistic, lauritzen1996graphical}. 

% Graphical Model Definitions
We consider a set of random variables $X := (X_j)_{j \in V}$ with a probability distribution taking values in a sample space $\mathcal{X}_{V}$ over the nodes in $G$. Suppose that a random vector $X$ has a joint probability density function $P(G) = P(X_1, X_2, ..., X_p)$. For any subset $S$ of $V$, let $X_{S} :=\{X_j : j \in S \subset V \}$ and $\mathcal{X}_{S} := \times_{j \in S} \mathcal{X}_{j}$ where $\mathcal{X}_{j}$ is a sample space of $X_j$. For any node $j \in V$, $\P(X_j \mid X_{S})$ denotes the conditional distribution of a variable $X_j$ given a random vector $X_{S}$. Then, a DAG model has the following factorization~\citep{lauritzen1996graphical}:  
\begin{equation}
\label{eq:factorization}
\P(G) = \P(X_1, X_2,..., X_p) = \prod_{j=1}^{p} \P(X_j \mid X_{\pa(j)}),
\end{equation}
where $\P(X_j \mid X_{\pa(j)})$ is the conditional distribution of $X_j$ given its parents variables $X_{\pa(j)} :=\{X_k : k \in \pa(j) \subset V \}$.
% $X_{\pa(j)}$. 
% .
 
% Problem settings including assumptions. 
We suppose that there are $n$ independent and identically distributed samples $X^{1:n} := ( X^{(i)} )_{i=1}^{n}$ from a given graphical model where $X^{(i)} := X_{1:p}^{(i)} = ( X_1^{(i)}, X_2^{(i)},\cdots, X_p^{(i)})$ is a $p$-variate random vector. The notation $\widehat{\cdot}$ denotes an estimate based on samples $X^{1:n}$. We also accept the causal sufficiency assumption that all important variables have been observed.

\subsection{Poisson DAG Model and its Identifiability}

\label{SecPoisson}

% Definition of Poisson DAG models
The definition of Poisson DAG models in \cite{park2015learning} is that each conditional distribution given its parents $X_j \mid X_{\pa(j)}$ is Poisson such that 
\begin{equation}
\label{eq:CondPDAGM}
X_j  \mid X_{\pa(j)} \sim \mbox{Poisson}(g_j(X_{\pa(j)})),
\end{equation}
where for any arbitrary positive link function $g_j: \mathcal{X}_{\pa(j)} \to \mathbb{R}^{+}$. Hence using the factorization in Equation~\eqref{eq:factorization}, the joint distribution is as follows:
\begin{equation}
\label{eq:JointPDAGM}
f_{G}(X) = \prod_{j \in V} f_j(X_j \mid X_{\pa(j)}).
\end{equation}
where $f_j$ is the probability density function of Poisson. 

A Poisson structural equation model (SEM) is a special case of a Poisson DAG model where the link functions $g_j$'s in Equation~\eqref{eq:CondPDAGM} are the standard log link function for Poisson generalized linear models (GLMs), i.e., $g_j(X_{\pa(j)}) = \exp(\theta_j + \sum_{k \in \pa(j)}{\theta_{jk}X_k})$ where $(\theta_{jk})_{k \in \pa(j)}$ represents the linear weights. Using factorization~\eqref{eq:factorization}, the joint distribution of a Poisson SEM can be written as:
\begin{align}
\label{eq:PDAGGLM}
&f(X_1,X_2,...,X_p) = \exp \Big( \sum_{j \in V} \theta_j X_j + \sum_{(k,j)\in E} \theta_{jk} X_j X_k - \sum_{j \in V} \log X_j ! - \sum_{j \in V} e^{ \theta_j + \sum_{k \in \pa(j)} \theta_{jk} X_k }\Big). 
\end{align}

Poisson DAG models have a useful moments relation for the identifiability:
\begin{proposition}
	\label{prop:moments ratio}
	Consider a Poisson DAG model~\eqref{eq:JointPDAGM} with non-degenerated rate parameter functions $(g_j(X_{\pa(j)}) )_{j \in V}$. Then, for any node $j \in V$, and any set $S_j \subset \nd(j)$, the following moments relation holds:
	 \begin{equation}
	 \label{eq:moment ratio}
	 	\frac{ \E( X_j^2) }{ \E\left[   \E( X_j \mid X_{S_j}) + \E( X_j \mid X_{S_j})^2 \right] } \geq  1
	 \end{equation}
	 Equivalently, 
	 \begin{equation*}
	 	\E( \var( \E(X_j \mid X_{\pa(j)} ) \mid X_{S_j} ) ) \geq 0. 
	 \end{equation*}
	 The equality only holds when $S_j$ contains all parents of $j$, that is, $\pa(j)  \subset S_j$. 
\end{proposition}

We include the proof in Section~\ref{SecPropProof1}. Proposition~\ref{prop:moments ratio} claims that when all parents of $j$, $\pa(j)$, contribute to its rate parameter, the moments ratio in Equation~\eqref{eq:moment ratio} is equal to 1 if a condition set $S_j$ contains all parents of $j$, $\pa(j) \subset S_j$, otherwise greater than 1. In Poisson SEMs, it is clear that the non-degenerated rate parameter function assumptions are equivalent to the \textit{non-zero coefficients} conditions, $|\theta_{jk}| > 0$ for all $k \in \pa(j)$ since $g_j(X_{\pa(j)}) = \exp(\theta_j + \sum_{k \in \pa(j)}{\theta_{jk}X_k})$. 

%This result is intuitively makes sense because when $S_j$ is a strict subset of $\pa(j)$, $\var( \E(X_j \mid X_{\pa(j)} ) \mid X_{S_j} ) = \var( g_j(X_{\pa(j)}) \mid X_{S_j} ) > 0 $ as long as $g_j(X_{\pa(j)})$ is non-degenerated. 

Now, we briefly explain how Poisson DAG models are identifiable from the moments ratio in Proposition~\ref{prop:moments ratio} using the bivariate Poisson DAG models illustrated in Fig.~\ref{figure1}: $G_1: X_1 \sim \mbox{Poisson}(\lambda_1), X_2 \sim \mbox{Poisson}(\lambda_2)$, where $X_1$ and $X_2$ are independent; $G_2: X_1 \sim \mbox{Poisson}(\lambda_1)$ and  $X_2 \mid X_1 \sim \mbox{Poisson}(g_2(X_1))$; and $G_3: X_2 \sim \mbox{Poisson}(\lambda_2)$ and $X_1 \mid X_2 \sim \mbox{Poisson}(g_1(X_2))$ for arbitrary non-degenerated positive functions $g_1, g_2: \mathbb{N} \cup \{0\} \to \mathbb{R}^{+}$. 

\begin{figure}[!t]
	\centering
	\begin {tikzpicture}[ -latex ,auto,
	state/.style={circle, draw=black, fill= white, thick, minimum size= 2mm},
	label/.style={thick, minimum size= 2mm}
	]
	\node[state] (X1)  at (0,0)   {\small{$X_1$} }; \node[state] (X2)  at (2,0)   {\small{$X_2$}}; \node[label] (X3) at (1,-.7) {$G_1$};
	\node[state] (Y1)  at (5,0)   {\small{$X_1$}}; \node[state] (Y2)  at (7,0)   {\small{$X_2$}}; \node[label] (Y3) at (6,-.7) {$G_2$};
	\node[state] (Z1)  at (10,0)   {\small{$X_1$}}; \node[state] (Z2)  at (12,0)   {\small{$X_2$}}; \node[label] (Z3) at (11,-.7) {$G_3$};
	\path (Y1) edge [shorten <= 2pt, shorten >= 2pt] node[above]  { } (Y2); 
	\path (Z2) edge [shorten <= 2pt, shorten >= 2pt] node[above]  { } (Z1);
\end{tikzpicture}
\caption{Bivariate directed acyclic graphs of $G_1$, $G_2$, and $G_3$.  }
\label{figure1}
\end{figure}
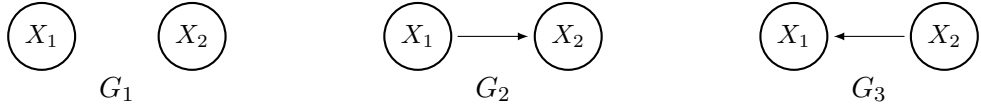

By Proposition~\ref{prop:moments ratio}, we can see that  $\E( X_j^2 )= \E(X_j) + \E(X_j)^2$ for all $j \in \{1, 2\}$ in $G_1$. In $G_2$, we can also see that 
$$
\E( X_1^2 )= \E(X_1) + \E(X_1)^2,\quad \text{and} \quad \E( X_2^2 ) > \E(X_2) + \E(X_2)^2.
$$
Similarly, in $G_3$, we have $\E( X_1^2 ) > \E(X_1) + \E(X_1)^2$, while $\E( X_2^2 ) = \E(X_2) + \E(X_2)^2$. Hence, we can determine the true graph based on the moments ratio $\E( X_j^2 ) /( \E(X_j) + \E(X_j)^2 )$. 

This idea of a moments relation in Proposition~\ref{prop:moments ratio} can easily apply to general p-variate Poisson DAG models, and hence, the models are identifiable by testing whether the moments ratio in Equation~\eqref{eq:moment ratio} is equal to 1 or greater than 1.

\begin{theorem}
	\label{ThmMainIdentifiability}
	Consider a Poisson DAG model~\eqref{eq:JointPDAGM} with rate parameters $(g_j(X_{\pa(j)}))_{j \in V}.$ If for any $j \in V$, rate parameter $g_j(\cdot)$ is non-degenerated, the Poisson DAG model is identifiable.  
\end{theorem}

%%% Comparison to ODS in population %%%
We include the proof in Section~\ref{SecTheo}. Theorem~\ref{ThmMainIdentifiability} claims that any Poisson DAG model is identifiable if all parents of node $j$ contribute to its rate parameter. Hence, Theorem~\ref{ThmMainIdentifiability} shows that any Poisson SEM is identifiable under the non-zero coefficients condition, $|\theta_{jk}| > 0$ for all $k \in \pa(j)$. This condition is also commonly assumed in (Gaussian) linear structural equation models for the model identifiability~\citep{spirtes1995directed, ghoshal2017learning,ghoshal2018learning,loh2014high,peters2014identifiability, park2019identifiability}. We believe that it is a natural condition that is in accordance with the intuitive understanding of relationships among variables. 

Our identifiability condition is strictly milder than the previous identifiability result in \cite{park2015learning} that is equivalent to $\var( \E( X_j \mid X_{\pa(j)})  \mid X_{S_j} = x ) > 0$ for all $x \in \mathcal{X}_{S_j}$ when $\pa(j) \not \subset S_j$ . For a better comparison, we consider a fully connected graph where $X_1 \sim \text{Poisson}( \lambda )$, $X_2 \mid X_1 \sim \text{Poisson}( \lambda + X_1)$, and $X_3 \mid X_1, X_2 \sim \text{Poisson}( \lambda + X_2 \mathbf{1}(X_1 \neq 0))$ where $\lambda$ is a positive constant and $\mathbf{1}(\cdot)$ is an indicator function. In this case, we can see $\var( \E( X_3 \mid X_{1}, X_{2})  \mid X_1 = 0 ) = 0$, and hence, the identifiability condition in~\cite{park2015learning} is not satisfied, while our condition is satisfied. 
%
%Compared to the previous identifiability condition in \citet{park2015learning}, our identifiability condition is strictly milder because the identifiability condition in \citet{park2015learning} is $\var( \E( X_j \mid X_{\pa(j)})  \mid X_{S_j} = x ) > 0$ for all $x \in \mathcal{X}_{S_j}$, while we require $\var( \E( X_j \mid X_{\pa(j)})  \mid X_{S_j} = x ) ) > 0$ for some $x \in \mathcal{X}_{S_j}$, and that is implied by non-degenerated $g_j(\cdot)$ condition by Proposition~\ref{prop:moments ratio}. 

%%% Comparison to ODS in sample %%%
In a Poisson SEM, the identifiability assumption in  \cite{park2015learning} is also satisfied under the non-zero coefficients condition. However, in the finite sample setting, the difference of both assumptions gets more crucial. For a positive constant $c$, \cite{park2015learning} requires $\min_{ x \in \mathcal{X}_{S_j} } \var( \E( X_j \mid X_{\pa(j)})  \mid X_{S_j} = x ) ) >  c$, while we need $\E( \var( \E( X_j \mid X_{\pa(j)})  \mid X_{S_j} ) ) >  c$.  Hence, our new identifiability assumption makes learning Poisson SEMs easier. We discuss this more in Section~\ref{SecTheo}. 

\subsection{Comparison to Poisson MRF}

\label{SecComparison}

In this section, we compare Poisson DAG models and MRFs where the conditional distributions of each node given its parents and neighbors are Poisson, respectively. To simplify the comparison, we consider the joint distribution of a Poisson SEM in Equation~\eqref{eq:PDAGGLM}. This is a form similar to the joint distribution of Poisson MRFs in~\cite{yang2015graphical}, where the joint distribution has the following form:
\begin{align}
\label{eq:PoissonMRF}
&f(X_1,X_2,...,X_p) = \exp \Big( \sum_{j \in V} \theta_j X_j + \sum_{(k,j)\in E} \theta_{jk} X_j X_k- \sum_{j \in V} \log X_j ! - A(\theta) \Big), 
\end{align}
where $A(\theta)$ is the log of the normalization constant. The key difference between a Poisson SEM and a Poisson MRF is the normalization constant $A(\theta)$ in Equation~\eqref{eq:PoissonMRF}, as opposed to the term $\sum_{j \in V} e^{ \theta_j + \sum_{k \in \pa(j)} \theta_{jk} X_k }$ in Equation~\eqref{eq:PDAGGLM}, which depends on variables.

% Limits of Poisson MRFs
\cite{yang2015graphical} proves that a Poisson MRF~\eqref{eq:PoissonMRF} is normalizable if and only if all $(\theta_{jk})$ values are less than or equal to $0$. This means Poisson MRFs only capture negative dependency relations. In addition, \cite{yang2015graphical} addresses the learning Poisson MRFs when the functional form of dependencies is $X_j  \mid X_{V \setminus j} \sim \mbox{Poisson}(\exp(\theta_j + \sum_{k \in \mathcal{N}(j)}{\theta_{jk}X_k}))$ where $\mathcal{N}(j)$ denotes the neighbors of a node $j$ in the graph.

% Strong Point: High Dimensional MRF
While Poisson MRFs have strong restrictions on the functional form for dependencies and the parameter space, they can be successfully learned in the high-dimensional settings with less restrictive constraints of sparsity. \cite{yang2015graphical} shows that Poisson MRFs can be recovered via $\ell_1$-regularized regression if $n = \Omega\left(d_{m}^2 \log^3 p \right)$, where $d_{m}$ is the degree of the undirected graph. In contrast, \cite{ park2017learning} shows that Poisson DAG models can be learned via the ODS algorithm if $n = \Omega( \max \{ d_{m}^4 \log^{12} p, \log^{5+d} p \} )$ where $d_{m}$ is obtained by the moralized graph and $d$ is the maximum indegree of the graph. This big difference in the sample complexity primarily comes from the unknown functional form for the dependencies in Poisson DAG models. In the next section, we will show that a significant advantage can be achieved by assuming the parametric function for the dependencies in terms of recovering the graphs. 

\section{Algorithm}

\label{SecAlgorithm}

%%%% Algorithm Overview %%%%
Here, we present our Moments Ratio Scoring (MRS) algorithm for learning the identifiable Poisson SEM~\eqref{eq:PDAGGLM}. Our algorithm alternates between an element-wise ordering search using the (conditional)  moments ratio, and a parent search using $\ell_1$-regularized GLM. Hence, the algorithm chooses a node for the first element of the ordering, and then determines its parents. The algorithm iterates this procedure until the last element of the ordering and its parents are determined.

Without loss of generality, assume that $\pi = (1,2,\cdots,p)$ is the true ordering.
Then Poisson SEMs~\eqref{eq:PDAGGLM} have the conditional distribution of $X_j$ given that all variables before $j$ in the ordering are reduced to the following Poisson GLM:
\begin{align}
\label{eq:TrueConditionalDistribution}
& P(X_{j} \mid X_{1:(j-1)} ) = \exp \bigg\{  \theta_j X_j + \sum_{k \in {1:(j-1)}} \theta_{jk} X_{k} X_j  + \log X_j ! - \exp \bigg( \theta_j + \sum_{k \in {1:(j-1)}} \theta_{jk} X_{k} \bigg) \bigg\}, 
\end{align}
where $\theta_{jk} \in \mathbb{R}$ represents the influence of node $k$ on node $j$. For ease of notation, let $\theta(j)$ be a set of parameters related to Poisson GLM~\eqref{eq:TrueConditionalDistribution}. Then $\theta(j) = (\theta_j, \theta_{\setminus j} ) \in \mathbb{R} \times \mathbb{R}^{j-1}$ where $\theta_{\setminus j}  = ( \theta_{jk} )_{k \in \{1,2,...,j-1\}}$ is a zero-padded vector with non-zero entries if $k \in \pa(j)$.

%%%% How to choose the j th element of the ordering. %%%%
Our MRS (Algorithm~\ref{MRS_Algorithm}) involves learning the ordering by comparing moments ratio scores of nodes using the following equations: 
\begin{align}
\label{EqnTruncScorej}
\widehat{\mathcal{S}}(1,j) &:=
\frac{ \widehat{\E}( X_j^2)}{ \widehat{\E}( X_j) + \widehat{\E}( X_j )^2  } \quad \text{and} \quad  
\widehat{\mathcal{S}}(m,j) :=\frac{ \widehat{\E}( X_j^2 )}{ \widehat{\E}\big( \widehat{\E}( X_j \mid X_{\widehat{\pi}_{1:(m-1)}}) + \widehat{\E}( X_j \mid X_{\widehat{\pi}_{1:(m-1)}})^2 \big)  }, 
\end{align}
where $\widehat{\pi}_{1:m} = \{\widehat{\pi}_1, ... ,\widehat{\pi}_m \}$,  $\widehat{\E}(X_j) = \frac{1}{n} \sum_{i = 1}^{n} X_j^{(i)}$, and $\widehat{\E}( \widehat{\E}(X_j \mid X_S) ) = \frac{1}{n} \sum_{i = 1}^{n} \exp \big(\widehat{\theta}_j^S + \sum_{k \in S} \widehat{\theta}_{jk}^S X_{k}^{(i)} \big)$, and $\widehat{\E}( \widehat{\E}(X_j \mid X_S)^2 ) = \frac{1}{n} \sum_{i = 1}^{n} \exp \big(2\widehat{\theta}_j^S + 2\sum_{k \in S} \widehat{\theta}_{jk}^S X_{k}^{(i)} \big)$ where $\widehat{\theta}_S(j) = (\widehat{\theta}_j^S, \widehat{\theta}_{\setminus j}^S)$ is the solution of the following $\ell_1$-regularized GLM:

\begin{align}
\label{eqn:theta_jS2}
\widehat{\theta}_S(j) := \arg \min \frac{1}{n} \sum_{i = 1}^{n} \bigg[  -X_j^{(i)} \bigg( \theta_j + \sum_{k \in S} \theta_{jk} X_{k}^{(i)} \bigg) + \exp \bigg(\theta_j + \sum_{ k \in S } \theta_{jk} X_{k}^{(i)} \bigg)  \bigg] + \lambda_j \sum_{k \in S} |\theta_{jk}|. 
\end{align}

%%%% Interpretation of scores %%%%
This score is an estimator of the moments ratio relation in Equation~\eqref{eq:moment ratio}. Hence, the correct element of the ordering has a score of 1, otherwise strictly greater than 1 in population. The ordering is determined one node at a time by selecting the node with the smallest score. Similar strategies of element-wise ordering learning can be found in many existing algorithms (e.g.,~\citealp{shimizu2011directlingam, ghoshal2017learning, ghoshal2018learning,  drton2018causal}).

%%%% Novelty of the Algorithm %%%%
The novelty of our algorithm is learning an ordering by testing which nodes have the smallest moments ratio in Equation~\eqref{eq:moment ratio} using the $\ell_1$-regularized GLM. By substituting the estimation of parameters $\theta(j)$ for an estimation of the conditional mean, we gain significant computational and statistical improvements compared to the previous works in~\cite{park2015learning, park2017learning} where the method of moments is used for estimating the conditional mean and variance. 

In principle, the number of conditional variances exponentially grows in the number of conditioning variables. Hence, if a conditioning set contains d-variables with 10 possible outcomes, then the number of possible computations is $10^d$. In other words, the minimum sample size for the ODS algorithm to be implemented is possibly $10^d$, otherwise, none of conditional variances can be estimated. %In Section~\ref{SecNume}, we empirically verify that it can be a critical issue for the ODS algorithm when a graph is not so sparse $(d =5)$. 

%%%% Finding Parents Set %%%%
As we discussed, the problem of a learning directed graph structure is the same as the problem of an learning undirected graph structure if the ordering is known. Hence, given the estimated ordering, the parents of each node $j$ can be learned via $\ell_1$-regularized GLM (see details in \citealp{meinshausen2006high, wainwright2006high, ravikumar2011high,yang2015graphical}). Therefore, we determine the estimated parents of a node $j$ as $\widehat{\pa}(j) := \{ k \in S: \widehat{\theta}_{j k}^S \neq 0 \}$ where $S = \widehat{\pi}_{1:(j-1)}$ and $\widehat{\theta}_S(j)$ is the solution to Equation~\eqref{eqn:theta_jS2}. %\textbf{We define the estimated parents of a node $j$ is $\widehat{\pa}(j) := \{ k \in \pi_{1:(j-1)}: \hat{\theta}_{jk} \neq 0 \}$. }

%%%% Algorithm: MRS %%%%
\setlength{\algomargin}{0.5em}
\begin{algorithm}[!t]
	\caption{ \bf Moments Ratio Scoring (MRS)~\label{MRS_Algorithm} }
	\SetKwInOut{Input}{Input}
	\SetKwInOut{Output}{Output}
	\SetKwInOut{Return}{Return}
	\Input{ $n$ i.i.d. samples, $X^{1:n}$}
	\Output{ Estimated ordering $\widehat{\pi}=(\widehat{\pi}_1,...,\widehat{\pi}_p)$ and an edge structure, $\widehat{E} \subset V \times V$}
	\BlankLine
	Set $\widehat{\pi}_{0} = \emptyset$\;
	\For{$m = \{1,2,\cdots,p\}$}{
		Set $S = \{\widehat{\pi}_1,\cdots,\widehat{\pi}_{m-1}\} $;\\
		\For{$j \in \{1,2,\cdots,p\} \setminus S$ }{
			Estimate $\widehat{\theta}_{S}(j)$ for $\ell_1$-regularized generalized linear model~\eqref{eqn:theta_jS2}; \\
			Calculate scores $\widehat{\mathcal{S}}(m,j)$ using Equation ~\eqref{EqnTruncScorej};\\
		}
		The $m^{th}$ element of the ordering, $\widehat{\pi}_m = \arg \min_j \widehat{\mathcal{S}}(m,j)$\;
		The parents of the $m^{th}$ element of the ordering, $\widehat{\pa}(\widehat{\pi}_m) = \{k \in S \mid \widehat{\theta}_{\widehat{\pi}_m k}^S \neq 0 \} $\;
	}
	\Return{Estimate the edge set, $\widehat{E} =  \cup_{m \in V} \{ (k, \widehat{\pi}_m) \mid k \in \widehat{\pa}(\widehat{\pi}_m) \} $}
\end{algorithm} 

\subsection{Computational Complexity}

\label{SecComp}

%%%% Computational Complexity for Our Algorithm %%%%: 
The computational complexity for the MRS algorithm involves the $\ell_1$-regularized GLM algorithm~\citep{friedman2009glmnet} where the worse-case complexity is $O(n p)$ for a single $\ell_1$-regularized regression run. More precisely, the coordinate descent method updates each gradient in $O(p)$ operations. Hence, with $d$ non-zero terms in the GLM, a complete cycle costs $O(p d)$ operations if no new variables become non-zero, and costs $O(n p)$ for each new variable entered (see details in \citealp{friedman2010regularization}). Since our algorithm has $p$ iterations and there are $p-j+1$ regressions with $j-1$ features for the $j$th iteration, the total worst-case complexity is $O(n p^3)$.

%%%% Comparison to MRF %%%%
The estimation of a Poisson MRF also involves a node-wise $\ell_1$-regularized GLM over all other variables, and hence the worse-case complexity is $O(n p^2)$ if the coordinate descent method is exploited. The addition of estimation of ordering makes $p$ times more computationally inefficient than the standard method for learning Poisson MRFs. 

%%%% Comparison to PC, GES %%%%
Learning a DAG model is NP-hard in general~\citep{chickering1994learning}. Hence, many state-of-the-art MEC and DAG learning algorithms, such as PC~\citep{spirtes2000causation}, GES~\citep{chickering2003optimal}, and MMHC~\citep{tsamardinos2006max}, are inherently greedy search algorithms. In the numerical experiments in Section~\ref{SecNume}, we compare MRS to greedy hill-climbing search-based GES and MMHC algorithms in terms of run time, and show that MRS has a significantly better computational complexity. 

\subsection{Theoretical Guarantees}

\label{SecTheo}

%%%% Overview %%%%
In this section, we provide theoretical guarantees on the MRS algorithm for learning Poisson SEMs~\eqref{eq:PDAGGLM}. The main result is expressed in terms of the triple $(n,p,d)$, where $n$ is a sample size, $p$ is a graph node size, and $d$ is the indegree of a graph.

\subsubsection{Assumptions}

%%%% Required Conditions for GLM %%%%
We begin by discussing the assumptions we impose on Poisson SEMs. Since we apply $\ell_1$-regularized regression for the parent selection, most assumptions are similar to those imposed in~\cite{wainwright2006high}, \cite{ravikumar2011high}, \cite{yang2015graphical} and \cite{park2017learning} where $\ell_1$-regularized regression was used for graphical model learning. 

%% Hessian Conditions %%
Important quantities are the Hessian matrices of the negative conditional log-likelihood of a node $j$ given some subsets of the nodes in the ordering, $S_j \in \{ \{\pi_{1} \}, \{\pi_{1}, \pi_{2} \},..., \{\pi_1,...\\,\pi_{j-1}\}  \}$. Let $Q^{j,S_j} := \bigtriangledown^2 \ell_j^{S_j}(\theta_{S}^{*}(j); X^{1:n})$ where 
\begin{align}
\label{eq:likelihood}
& \ell_j^{S_j}( \theta_{S_j}(j), X^{1:n} ) := \frac{1}{n} \sum_{i = 1}^{n} \bigg[  -X_j^{(i)} \bigg( \theta_j^{S_j} + \sum_{k \in S_j} \theta_{jk}^{S_j} X_{k}^{(i)} \bigg) + \exp \bigg(\theta_j^{S_j} + \sum_{ k \in S_j } \theta_{jk}^{S_j} X_{k}^{(i)} \bigg)  \bigg],
\end{align} 

\begin{align}
\label{ThetaSj}
& \theta_{S_j}^{*}(j) := \arg \min \E \bigg[ - X_j \bigg( \theta_j^{S_j} + \sum_{k \in S_j} \theta_{jk}^{S_j} X_{k} \bigg) + \exp \bigg(\theta_j^{S_j} + \sum_{ k \in S_j } \theta_{jk}^{S_j} X_{k} \bigg)  \bigg].
\end{align}
%where $\Theta_{S_j}(j):= \{ \theta_{S_j}(j) \in \mathbb{R}^{|S_j|+1} : \theta_{jk}^{S_j} = 0 \;\; \textrm{for} \;\; k \notin \pa(j) \}$.

For ease of notation, we define a set for the non-zero elements of $\theta_{S_j}^{*}(j)$, 
\begin{align}
\label{eq:subparents}
 T_j := \{k \in S_j \mid \theta_{jk}^{*} \neq 0 \quad \text{where } \theta_{S_j}^{*}(j) = (\theta_{j}^{*}, \theta_{jk}^{*}) \}.
\end{align}
We note that if $S_j$ contains all parents of $j$, $\pa(j) \subset S_j$, then $T_j = \pa(j)$. Lastly, for simplicity, we let $A_{S S}$ denote the $|S| \times |S|$ sub-matrix of the matrix $A$ corresponding to variables $X_{S}$.

%https://en.wikipedia.org/wiki/Min-max_theorem
%https://math.stackexchange.com/questions/1670000/eigenvalues-of-the-principal-submatrix-of-a-hermitian-matrix/1670001#1670001

\begin{assumption}[Dependence Assumption]
	\label{A1Dep}
	For any $j \in V$ and any $S_j \in \{ \{\pi_{1} \}, \{\pi_{1}, \pi_{2} \},\\ ..., \{\pi_1,...,\pi_{j-1}\}  \}$, there exist positive constants $\rho_{\min}$ and $\rho_{\max}$ such that 
	\begin{align*}
	&\min_{j \in V} \lambda_{\min}\left(Q^{j,S_j}_{ T_j T_j }\right)  \geq \rho_{\min}, 
	\quad \text{ and } \quad  
	\max_{j \in V }  \lambda_{\max} \left( \frac{1}{n}\sum_{i=1}^{n} X_{\pa(j)}^{(i)} (X_{\pa(j)}^{(i)})^T \right) \leq \rho_{\max},
	\end{align*}
	where $T_j$ is in Equation~\eqref{eq:subparents}, $\lambda_{\min}(A)$ and $\lambda_{\max}(A)$ are the smallest and largest eigenvalues of the matrix $A$, respectively. 
\end{assumption}

\begin{assumption}[Incoherence Assumption]
	\label{A2Inc}
	For any $j \in V$ and any $S_j \in \{ \{\pi_{1} \}, \{\pi_{1}, \pi_{2} \},\\..., \{\pi_1,...,\pi_{j-1}\}  \}$, there exists a constant $\alpha \in (0,1]$ such that 
	$$\max_{j, S_j } \max_{t \in T_j^c} \| Q_{t T_j}^{j,S_j} (Q_{T_j T_j}^{j,S_j})^{-1}\|_1 \leq 1 - \alpha,$$
	where $T_j$ is in Equation~\eqref{eq:subparents}. 
\end{assumption} 

Assumption~\ref{A1Dep} ensures that the parent variables are not too dependent. In addition, Assumption~\ref{A2Inc} ensures that parent and non-parent variables are not highly correlated. These two assumptions are standard in all neighborhood regression approaches to variable selection involving $\ell_1$-regularized based methods, and these conditions have imposed in proper works for both high-dimensional regression and graphical model learning. 

To control the tail behavior of likelihood functions, we require a bounded sample assumption which is also imposed in the standard $\ell_1$-regularized Poisson regression (e.g., \citealp{jia2017sparse}).

\begin{assumption}[Bounded Sample Assumption]
	\label{A3Con}
	For any $i \in \{1,2,...,n\}$, $j \in V$, and for all $S_j  \in \{ \{\pi_{1} \}, \{\pi_{1}, \pi_{2} \},..., \{\pi_1,...,\pi_{j-1}\}  \}$, the samples are bounded: 
	$$
	\max_{i, j} \{ X_{j}^{(i)} \} < C_{x} \log( \max\{n,p\} ) \quad \text{and} \quad \max_{i, j} \{ \exp( \theta_{j}^* + \sum_{k \in S_j} \theta_{jk}^* X_k^{(i)}   ) \} < C_{x} \log( \max\{n,p\} ).
	$$
	where $C_{x} > 2$ is a positive constant. 
\end{assumption}

Assumption~\ref{A3Con} is closely related to the rate parameters. For instance, the rate parameter of $X_j^{(i)}$ is $\exp( \theta_{j}^* + \sum_{k \in \pa(j)} \theta_{jk}^* X_k^{(i)} )$ by the definition of Poisson SEMs. Hence, Assumption~\ref{A3Con} can be understood that too large rate parameters, that leads to a large value of a sample, are not allowed for all conditional distributions.

In fact, Assumption~\ref{A3Con} is satisfied with a high probability when $(\theta_{jk}^*)$ are negative. Since the second condition in Assumption~\ref{A3Con} is directly satisfied with negative $(\theta_{jk}^*)$, we discuss the first condition: Using the union bound,
$$
P\left( \max_{i,j} X_j^{(i)} \geq C_x  \log( \max\{n,p\} ) \right) \leq n.p \max_{i,j} \frac{ \E( \exp(X_j^{(i)}) ) }{ ( \max\{n,p\} )^{C_x}  } \leq  \max_{i,j}  \frac{ \E( \exp(X_j^{(i)}) ) }{ ( \max\{n,p\} )^{C_x - 2}  }. 
$$ 
In addition, the moment generating function is bounded when $(\theta_{jk}^*)$ are negative.
$$ 
\E( \exp(X_j)) \leq \E( \E( \exp(X_j) \mid X_{\pa(j)} ) ) \leq \E( \exp( \theta_j^* + \sum \theta_{jk}^* X_k))  \leq  \exp( \theta_j^* ).
$$
Hence, given the negative  $(\theta_{jk}^*)$ assumption, Assumption~\ref{A3Con} is satisfied with probability at least $1 -  \max_{j} \exp({\theta_j} )/   ( \max\{n,p\} )^{C_x - 2}$. 
%We emphasize that all of our assumptions can be satisfied with some positive $\theta_{jk}^*$ with a high probability, and we provide an example in Section~\ref{SecStar}. 

%%%% Required Conditions for mean-variance relations %%%%
Lastly, we require a stronger version of the moments ratio relation in Equation~\eqref{eq:moment ratio}, because we move from the population to the finite samples. This assumption only involves learning the ordering of a graph.  

\begin{assumption}
	\label{A1}
	For all $j \in V$ and $S_j \in \{ \{\pi_{1} \}, \{\pi_{1}, \pi_{2} \},..., \{\pi_1,...,\pi_{j-1}\}  \}$, there exists an $M_{\min} > 0$ such that
	$$
	\E( X_j^2)  > ( 1 + M_{\min} ) \E[ \E( X_j \mid X_{S_j}) + \E( X_j \mid X_{S_j})^2 ].
	$$
\end{assumption}

%%%% Direct Comparision to Poisson MRF %%%%
Now, we compare Assumptions~\ref{A1Dep},~\ref{A2Inc},~\ref{A3Con}, and~\ref{A1} to the assumptions for learning Poisson MRFs and DAG models. As discussed, our assumptions are similar to the assumptions in \cite{yang2015graphical} and \cite{park2017learning} since all methods exploit the $\ell_1$-regularized GLM.
However, the assumptions in \cite{yang2015graphical} only involve neighbors of node $j$, that is, $S_j = V \setminus j$. While our assumptions involve some subsets of parents, that is, $S_j \in \{ \{\pi_{1} \}, \{\pi_{1}, \pi_{2} \},..., \{\pi_1,...,\pi_{j-1}\}  \}$ due to the unknown ordering. In addition, they do not assume the bounded sample assumption. However, they assume the restricted negative parameter space $\theta_{jk} < 0$ due to the normalizability issue. As we explained, if all parameters are negative in a Poisson SEM, the moment generating function is bounded, and hence, the bounded sample assumption is satisfied with a high probability. Lastly, \cite{yang2015graphical} does not have the moments ratio assumption, since it is only used for recovering the ordering.

%%%%%%%%%%%%%%%%%%%%%%%%%%%%%%
We compare the required assumptions for the MRS and ODS algorithms in \cite{park2017learning}. A major difference is that the MRS algorithm directly estimates the graph, while the ODS algorithm estimates the moralized graph to reduce the search space of DAGs, and then, estimates the graph. Hence, our assumptions involve some parents of node $j$, while their assumptions involve not only parents, but neighbors of node $j$, that is, $S_j = \{ \{\pi_1,...,\pi_{j-1}\}, V \setminus j \}$. In addition, they require a sparse moralized graph and adjacent faithfulness that are also known to be restrictive. We note that the sparse moralized graph assumption can be very strong since a sparse moralized graph is not implied by a sparse graph. For instance, consider a star graph where $X_1 \to X_{j}$ for all $j \in \{2,3,...,p\}$ in Fig.~\ref{fig:star}. This star graph has the maximum degree of the moralized graph is $p-1$, while the maximum indegree is 1. 

Another major difference is in the moments ratio assumption. More precisely, \cite{,park2015learning,park2017learning} assume $\var( \E(X_j \mid X_{S} = x)) > c$ for all $x \in \mathcal{X}_{S}$ when $\pa(j) \not \subset S$,  while we require $\E( \var( \E(X_j \mid X_{S} = x)) ) > c$. To emphasize the difference, we consider a 3-node graph $X_1 \to X_2 \to X_3$ where $X_1 \sim Poisson(\lambda), X_2 \mid X_1 \sim Poisson( \exp( \theta_1 X_1) )$, and $X_3 \mid X_2 \sim Poisson( \exp( \theta_2 X_2) )$. Then, for $j = 3$ and $S = 1$, we have 
$$
\var( \E(X_3 \mid X_2) \mid X_1 ) = \var( \exp(\theta_2 X_2) \mid X_1) < \E( \exp(2 \theta_2 X_2) \mid X_1) =  \exp( e^{\theta_1 X_1} ( e^{2 \theta_2} -1 ) ).
$$
Hence, for some constants $\theta_1, \theta_2$ and $c$, if $X_1 <  \frac{1}{\theta_1} ( \log \log c - \log( e^{2\theta_2} -1) )$, their assumption is not satisfied, while Assumption~\ref{A1} holds.  

Lastly, the ODS algorithm requires at least two distinct element of $X_{\pa(j)}^{(i)}$ for a conditional variance estimation, $\var(X_j \mid X_{\pa(j)} )$. In principle, it can be $2^d$ by assuming all variables are binary. Hence when $d$ is not so sparse, the ODS algorithm often fails to be implemented. In Section~\ref{SecNume}, we empirically verify that it can be a critical issue for the ODS algorithm when a graph is not so sparse $(d =5)$. 
%We provide some simulation results of the number of implementation failures with a random Poisson SEM with $d = 5$ in Section~\ref{SecNume}.
Therefore, we believe that the assumptions for the MRS algorithm are more realistic. 

%%%% Direct Comparison to Poisson DAG %%%%
Although our assumptions are standard in the previous works of $\ell_1$-regularized Poisson regressions, we have to note that the assumptions cannot be confirmed from data and they could be restrictive. However, they are not strong for $\ell_1$-regularized regression when samples are from Gaussian SEMs~(see e.g., \citealp{ravikumar2011high}). Hence, we conjecture that our assumptions can be satisfied with a high probability under mild conditions, and leave this to future study.

\subsubsection{Main Result}

%%%% Main result %%%%
Putting together Assumptions~\ref{A1Dep},~\ref{A2Inc},~\ref{A3Con}, and~\ref{A1}, we have the following main result that a Poisson SEM can be recovered via our MRS algorithm in high-dimensional settings. The theorem provides not only sufficient conditions, but also the probability that our method recovers the true graph structure.

\begin{theorem}%[Graph Structure Recovery]
	\label{ThmMainTheorm}
	Consider a Poisson SEM~\eqref{eq:PDAGGLM} with parameter vector $( \theta(j) )_{j \in V}$ and the maximum indegree of the graph $d$. 
	%%%% regularizer %%%%
	Suppose that the regularization parameter~\eqref{eqn:theta_jS2} is chosen, such that 
	$$
	 \frac{4 C_{x}^2 \sqrt{2}(2-\alpha) }{  \alpha} \frac{ \log^2 (\max\{n,p\}) }{ \kappa_1(n,p)  } \leq  \lambda_{j}  \leq \frac{ \alpha \rho_{\min}^2 }{ 10^2 C_{x}^2 (2 - \alpha) \rho_{\max} d \log^2 (\max\{n,p\})  },
	$$ 
	for any $\alpha = (0, 1]$, and $\kappa_1(n,p) \geq \frac{4 \sqrt{2} \cdot  10^2 C_{x}^4 \cdot(2-\alpha)^2 }{  \alpha^2} \frac{ \rho_{\max} }{ \rho_{\min}^2 }  d \log^4 (\max\{n,p\})$. 
	%%%% Assumptions %%%%
	Suppose also that Assumptions~\ref{A1Dep},~\ref{A2Inc},~\ref{A3Con} and~\ref{A1} are satisfied and the values of the parameters in Equation~\eqref{eq:PDAGGLM} are sufficiently large such that $\min_{(j,k) \in E} | \theta_{jk} | \geq \frac{10}{\rho_{\min}} \sqrt{d}\lambda_j$. Then, for any $\epsilon > 0$, there exists a positive constant $C_{\epsilon}$ such that if the sample size is sufficiently large $n > C_{\epsilon} (\kappa_1(n,p) )^2 \log p$, then the MRS algorithm uniquely recovers the graph with a high probability: %$	P( \widehat{G} = G ) \geq 1 - \epsilon.$
	\begin{align*}
	P( \widehat{G} = G ) \geq 1 - \epsilon.
	\end{align*}
\end{theorem} \vspace{-0.0cm}

Detailed proof is provided in Appendices~\ref{SecThmMainProof1} and \ref{SecMainThmProof2}. Appendix~\ref{SecThmMainProof1} provides the error probability that $\ell_1$-regularized regression recovers the true parents of each node given the true ordering, and Appendix~\ref{SecMainThmProof2} provides the error probability that $\ell_1$-regularized regression recovers the ordering. The key technique for the proof is that the \emph{primal-dual witness} method used in sparse regularized regressions and related techniques \citep{meinshausen2006high, wainwright2006high, ravikumar2011high, yang2015graphical}. Theorem~\ref{ThmMainTheorm} intuitively makes sense because neighborhood selection via the $\ell_1$-regularized regression is a well-studied problem, and its bias can be controlled by choosing the appropriate regularization parameter $\lambda_j$. Hence, our moments ratio scores can be sufficiently close to the true scores to recover the true ordering. 

%Both results allows the algorithm to recover a true graph $G$ in high dimensional settings. 

Theorem~\ref{ThmMainTheorm} claims that if $n = \Omega( d^{2} \log^9 p )$, our MRS algorithm recovers an underlying graph with a high probability. Hence, our MRS algorithm works in a high-dimensional setting, provided that the indegree of a graph $d$ is bounded. This sample bound result shows that our method has much more relaxed constraints on the sparsity of the graph than the previous work in~\cite{park2017learning}, where the sample bound is $n = \Omega( \max \{ d_{m}^4 \log^{12} p, \log^{5+d} p \} )$. Moreover, it also shows that learning Poisson DAG models may require more samples than the learning Poisson MRFs in~\cite{yang2015graphical}, where the sample bound is $n = \Omega( d_{m}^2 \log^3 p ) )$ due to the existence of the ordering and the unrestricted parameter space.   

\subsubsection{Poisson SEM with a Star Graph Example}

\label{SecStar}

In this section, we discuss the validity of our assumptions using a special Poisson SEM with the star graph in Fig.~\ref{fig:star} where $X_1 \sim Poisson(\lambda)$, $X_j \mid X_1 \sim Poisson( \exp(\theta X_1) )$, for $j \in \{2,3,...p\}$. This consists of a single hub node connected to the rest of nodes. With this star graph, we show that our assumptions can be satisfied with positive $(\theta_{jk})$. 

%, and how restrictive the sparse moralized graph assumption is

In order to discuss the validity of Assumptions~\ref{A1Dep},~\ref{A2Inc}, \ref{A3Con}, and \ref{A1} in this particular example, we first calculate the expectation of the Hessian matrix of Equation~\eqref{eq:likelihood}: For any $j \in \{2,3,...p\}$,
\begin{align*}
	\E( X_1^2 \exp( \theta X_1)  ) &= \frac{ \partial^2 }{ \partial \theta^2 } \E( \exp( \theta X_1) ) = \lambda \exp( \lambda ( \exp(\theta) -1 )  + \theta) (\lambda \exp(\theta) +1), \\
	\E( X_1 X_j \exp( \theta X_1)  ) &= \frac{ \partial }{ \partial \theta } \E( \exp( \theta_1 X_1) X_j ) =  \frac{ \partial }{ \partial \theta } \E( \exp( \theta X_1) E(X_j \mid X_1) ) \\
	&= \frac{ \partial }{ \partial \theta } \E( \exp( 2 \theta X_1) ) = 2 \lambda \exp( \lambda( \exp(2 \theta) -1 ) + 2 \theta).  
\end{align*}

Hence, the population version of Assumption~\ref{A1Dep} is reduced to
$$
	\rho_{\min} <  \lambda \exp( \lambda ( \exp(\theta) -1 )  + \theta) (\lambda \exp(\theta) +1) \quad \text{and} \quad \lambda + \lambda^2 < \rho_{\max}. 
$$

%https://www.wolframalpha.com/input/?i=2+*+exp(+2*+(exp(b)-1)+%2Bb)+*+(+2*+exp(b)+%2B+1)+%3E+0.1,+a+%3D%3D+2

It can be satisfied with some positive values of $\theta$. For $\lambda = 2$, $\rho_{\min} = 0.01$ and  $\rho_{\max} = 10$, Assumption \ref{A1Dep} is satisfied if $\theta > -3.426$. In addition, for $\lambda = 5$, $\rho_{\min} = 0.01$ and  $\rho_{\max} = 50$, it is also satisfied if $\theta > -2.205$.

\begin{figure}[!t]
	\centering
	\begin {tikzpicture}[ -latex ,auto,
	state/.style={circle, draw=black, fill= white, thick, minimum size= 2mm},
	label/.style={thick, minimum size= 2mm}
	]
	\node[state] (X1)  at (3,1.5)   {\small{$X_1$} }; 
	\node[state] (X2)  at (-2,0)  {\small{$X_2$}}; 
	\node[state] (X3)  at (0,0)   {\small{$X_3$}}; 
	\node[state] (X41)  at (2,0.0)   {\small{$\cdots$}}; 
	\node[state] (X42)  at (4,0.0)   {\small{$\cdots$}}; 
	\node[state] (X43)  at (6,0.0)   {\small{$\cdots$}}; 
	\node[state] (X5)  at (8,0)   {\small{$X_p$}}; 
	
	\path (X1) edge [shorten <= 2pt, shorten >= 2pt] node[above]  {} (X2); 
	\path (X1) edge [shorten <= 2pt, shorten >= 2pt] node[above]  { } (X3);
	\path (X1) edge [shorten <= 2pt, shorten >= 2pt] node[above]  { } (X41);
	\path (X1) edge [shorten <= 2pt, shorten >= 2pt] node[above]  { } (X42);
	\path (X1) edge [shorten <= 2pt, shorten >= 2pt] node[above]  { } (X43);
	\path (X1) edge [shorten <= 2pt, shorten >= 2pt] node[above]  { } (X5);
\end{tikzpicture}
\caption{Star graph example  }
\label{fig:star}
\end{figure}
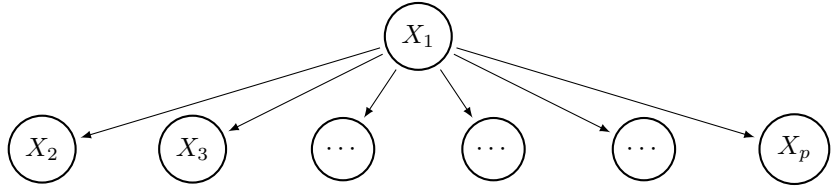

In addition, the population version of Assumption \ref{A2Inc} can be written as
$$
	\max_{j \in V \setminus \{1 \} } \max_{t \in V \setminus \{1, j\} } | \E( Q_{t 1}^{j,1}) \E( (Q_{1 1}^{j, 1})^{-1} ) | = \frac{2\cdot\exp ( \lambda \exp(\theta)( \exp(\theta) -1) + \theta) }{ \lambda \exp( \theta ) +1 }	\leq 1 - \alpha.
$$

This condition is also satisfied with positive values of $\theta$. For $\lambda = 2$ and $\alpha = 0.01$, a simple algebra yields that Assumption \ref{A2Inc} is satisfied if $\theta< 0.141$. % that is the solution of $1 + \exp( \theta ) - \exp( \theta - \exp(\theta) + \exp( 2\theta) )$. for $\lambda = 1$ and $\alpha = 1/2$, Assumption \ref{A2Inc} is satisfied if $\theta< 0$.
In addition, for $\lambda = 5$ and $\alpha = 0.01$, the assumption is satisfied if $\theta< 0.165$. 

%https://www.wolframalpha.com/input/?i=2+exp(a*+exp(b)*+(exp(b)-1)+%2Bb)+%2F+(+a*+exp(b)+%2B+1)+%3C+1,+a+%3D+2
 
In terms of Assumption~\ref{A3Con}, we also claim that it can be satisfied with some positive $\theta$. Since the moment generating function of $X_1$ is $\exp( \lambda(e -1) )$, we have, 
$$
 P( X_1^{(i)} > C_{x} \log( \max\{n,p\}) ) < \frac{ \E( \exp( X_1^{(i)} ) ) }{ \max\{n,p\}^{C_{x}} } =  \frac{ \exp( \lambda(e -1) ) }{ \max\{n,p\}^{C_{x} } }.
$$
where $C_{x} > 2$ is a positive constant in Assumption~\ref{A3Con}. 

For other nodes $j \in \{2,3,...,p\}$, we have, 
$$
	P( X_j^{(i)} > C_{x} \log( \max\{n,p\}) ) \leq \frac{ \E( \E( \exp( X_j^{(i)} ) \mid X_1^{(i)} ) ) }{ \max\{n,p\}^{C_{x}} }  = \frac{ \E( \exp( \exp( \theta X_1^{(i)})( e-1 ) ) ) }{ \max\{n,p\}^{C_{x}} }.
$$
 
For $\theta < \log X_1^{(i)} / X_1^{(i)}$, we have,  
$$
P( X_j^{(i)} > C_{x} \log( \max\{n,p\}) ) \leq \frac{ \E( \exp( X_1^{(i)} ( e-1 ) ) ) }{ \max\{n,p\}^{C_{x}} } =  \frac{ \exp( \lambda(e^{e-1} -1) ) }{ \max\{n,p\}^{C_{x} } }.
$$

Hence, for $\theta < \log ( C_{x} \log ( \max\{n,p\} ) )/ C_{x} \log ( \max\{n,p\} )$ that is the lower bound of $\log X_1^{(i)} / X_1^{(i)}$ given $X_1^{(i)} < C_{x} \log ( \max\{n,p\})$, Assumption~\ref{A3Con} is satisfied with a high probability:
$$
P\left( \max_{i,j} X_j^{(i)} > C_{x} \log( \max\{n,p\}) \right) \leq  \frac{ \exp( \lambda(e^{e-1} -1) ) }{ \max\{n,p\}^{C_{x}-2 } }.
$$

Now, we discuss Assumption~\ref{A1}. A simple calculation shows that, for any $j \in \{2,3,...,p\}$, 
$$
\E(X_j) =  \exp( \lambda ( \exp( \theta - 1 ))  ), \text{ and}
\quad
\E(X_j^2) = \exp( \lambda ( \exp( 2 \theta) - 1 )  ) +  \exp( \lambda ( \exp( \theta) - 1 )  ). 
$$

Hence, Assumption~\ref{A1} is equivalent to the constraint, 
$$
\exp( \lambda ( \exp( 2 \theta) - 1 )  ) > M_{min} \exp( \lambda ( \exp( \theta) - 1 )  ) + (1+ M_{min}) \exp( 2\lambda ( \exp( \theta ) - 1 )  ).
$$

This condition is also satisfied with some positive $\theta$. For $\lambda = 1$ and $M_{min} = 0$, as we discussed in Proposition~\ref{prop:moments ratio}, Assumption \ref{A2Inc} is always satisfied with any value of $\theta \neq 0$. For $\lambda = 2$ and $M_{min} = 0.001$, Assumption \ref{A2Inc} is satisfied if $|\theta| > 0.033$. Lastly, for $\lambda = 5$ and $M_{min} = 0.001$, Assumption \ref{A2Inc} is satisfied if $|\theta| > 0.021$. Therefore, we show that for this particular star graph, Assumption~\ref{A1Dep},~\ref{A2Inc},~\ref{A3Con}, and~\ref{A1} can be satisfied with a high probability by allowing positive $\theta$. 

%https://www.wolframalpha.com/input/?i=exp(+a+(+exp(+2+t)+-+1+)++)+-+M+exp(+a+(+exp(+t+)+-+1+)++)+-+(1%2B+M)+exp(+2+a+(+exp(+t+)-+1+)++)+%3E0,+M+%3D%3D+0,+a+%3D%3D+1

%이상함.. 
Finally, we emphasize that the sample complexity of the MRS algorithm, $n = \Omega( d^{2} \log^{9} p)$, does not rely on the maximum degree of the moralized graph, $d_m$, while many DAG learning algorithms using the sparsity of the moralized graph or Markov blanket inevitably depend on $d_m$. For the star graph with $d =1$ and $d_m = p-1$, the MRS algorithm requires $n = \Omega( \log^{9} p)$ to recover the graph in high dimensional settings, while the ODS algorithm may fail since its sample complexity is $\Omega( d_m^{4} \log^{12} p)$. This fact implies that, unlike the ODS algorithm, the MRS algorithm can recover a sparse graph containing hub nodes in high dimensional settings. %In other words, the sparse moralized graph assumption could be very restrictive. 

%we discuss the sample complexity of our MRS algorithm compared to the ODS algorithm in \citet{park2017learning}. 

\section{Numerical Experiments}

\label{SecNume}

%%%% Introduction %%%%
In this section, we provide simulation results to support our main theoretical results of Theorem~\ref{ThmMainTheorm} and the computational complexity in Section~\ref{SecComp}: (i) the MRS algorithm recovers the ordering and edges more accurately as sample size increases; (ii) the required sample size $n = \Omega( d^{2} \log^9 p )$ depends on the number of nodes $p$ and the complexity of the graph $d$; (iii) the MRS algorithm accurately learns the graphs in high-dimensional settings ($p > n$); and (iv) the computational complexity is $O(n p^3)$ at worst. We also show that the MRS algorithm performs favorably compared to the ODS~\citep{park2015learning}, GES~\citep{chickering2003optimal}, and MMHC~\citep{tsamardinos2006max} algorithms. In addition, we investigate how sensitive our MRS algorithm is to deviations from the assumption about the link functions by using the identity link function in Equation~\eqref{eq:JointPDAGM}. Lastly, we also investigate how well the MRS algorithm recovers undirected edges when samples are generated by Poisson and truncated Poisson MRFs~\citep{yang2013poisson, yang2015graphical, inouye2017review}. 

\subsection{Random Poisson SEMs}

\label{SecRanPSEM}

%%%% Simulation Settings %%%%
We conducted simulations using $200$ realizations of $p$-node Poisson SEMs~\eqref{eq:PDAGGLM} with the randomly generated underlying DAG structures while respecting the indegree constraints $d \in \{1, 5, 10\}$. A graph with $d=1$ is a special case where there is no v-structure, and therefore, the corresponding MEC is completely undirected. The set of non-zero parameters $\theta_{j}, \theta_{jk} \in \mathbb{R}$ in Equation~\eqref{eq:PDAGGLM} was generated uniformly at random in the range $\theta_{j} \in [1, 3]$, $\theta_{jk} \in [-1.5, -0.5] \cup [0.5, 1.5]$ for $d=1$, and $\theta_{jk} \in [-1, -0.1] \cup [0.1, 1]$ for $d = 5, 10$, which helps the generated values of samples to avoid either all zeros or from going beyond the maximum possible value of the R program ( $> 10^{309}$). Nevertheless, if some samples were beyond the maximum possible value, we regenerated the parameters and samples.

%%%% Comparison Method ODS %%%%
The MRS and ODS algorithms were implemented using $\ell_1$-regularized likelihood where we used five-fold cross validation to choose the regularization parameters. Where mean squared error was within two standard error of the minimum mean squared error, we chose the minimum value for the moments ratio scores and the largest value for parent selection. That was because a less biased estimator is preferred for the score calculation, and we preferred a sparse graph containing only legitimate edges. We acknowledge that the level of sparsity can be adjusted according to the importance of precision or recall.

% while considering sparsity of the graph. 
%The level of sparsity can be easily adjusted according to the importance of precision or recall. 

%%%% Evaluation Measure, Comparison Method GES, MMHC %%%%
In Fig.~\ref{fig:result001}, we compare the MRS algorithm to state-of-the-art ODS, GES and MMHC algorithms for graph node size $p = \{20, 200\}$, varying sample size $n \in \{25, 50, ..., 250\}$ for $d = 1$ and $n = \{100, 200,..., 1000\}$ for $d = 10$, and provide two results: (i) the average precision ($\frac{ \# \text{ of correctly estimated edges } }{ \# \text{ of estimated edges} }$); (ii) the average recall ($\frac{ \# \text{ of correctly estimated edges } }{ \# \text{ of true edges} }$). As discussed, the both GES and MMHC algorithms only recover the partial graph by leaving some arrows undirected. Therefore, we also provide average precision and recall for the estimated MECs in Fig.~\ref{fig:result002}. Lastly, we provide an oracle, where the true parents of each node are used, while the ordering is estimated via $\ell_1$-regularized GLM. Hence, we can see where the errors come from between the ordering estimation or parent selection. We considered more parameters ($\theta_{jk}, n, p, d$), but for brevity, we focus on these settings.

\begin{figure*}[!t]
	\centering \hspace{-2mm}
	\begin{subfigure}[!htb]{.22\textwidth} 
		\includegraphics[width=\textwidth, trim={0.5cm 0.2cm 0 0.2cm},clip]{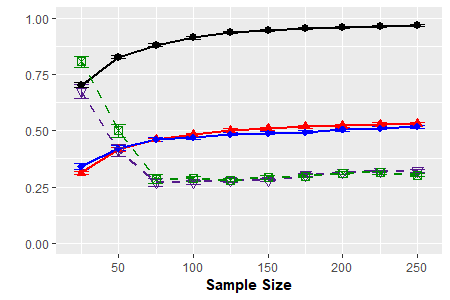}
		\caption{Prec:$p$=20,$d$=1}
	\end{subfigure}
	\begin{subfigure}[!htb]{.22\textwidth}
		\includegraphics[width=\textwidth, trim={0.5cm 0.2cm 0 0.2cm},clip]{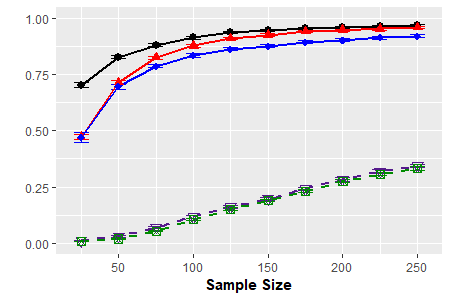}
		\caption{Reca:$p$=20,$d$=1}
	\end{subfigure}
	\begin{subfigure}[!htb]{.22\textwidth} 
		\includegraphics[width=\textwidth, trim={0.5cm 0.2cm 0 0.2cm},clip]{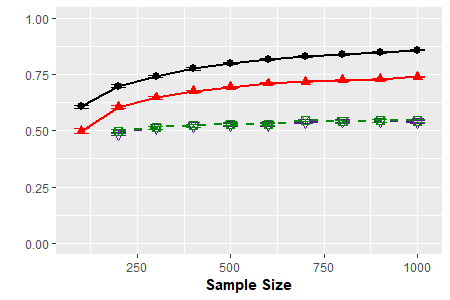}
		\caption{Prec:$p$=20,$d$=10}
	\end{subfigure}
	\begin{subfigure}[!htb]{.22\textwidth}
		\includegraphics[width=\textwidth, trim={0.5cm 0.2cm 0 0.2cm},clip]{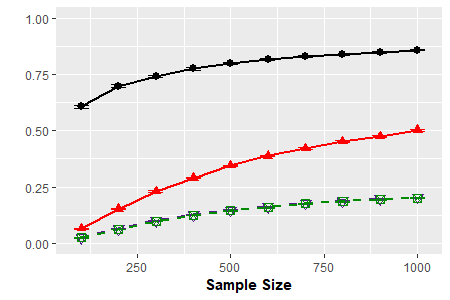}
		\caption{Reca:$p$=20,$d$=10}
	\end{subfigure}	
	\begin{subfigure}[!htb]{.07\textwidth}
		\includegraphics[width=\textwidth, trim={14cm 0 0.3cm 1.5cm},clip]{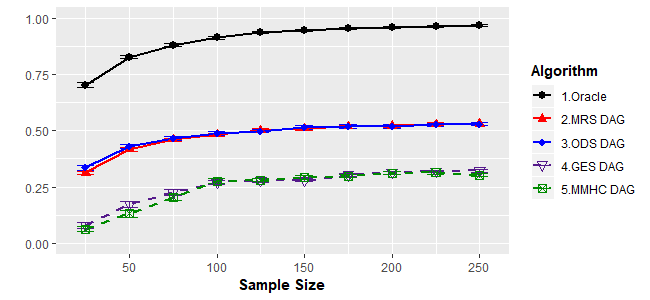}
	\end{subfigure}
	
	\begin{subfigure}[!htb]{.22\textwidth} 
		\includegraphics[width=\textwidth, trim={0.5cm 0.2cm 0 0.2cm},clip]{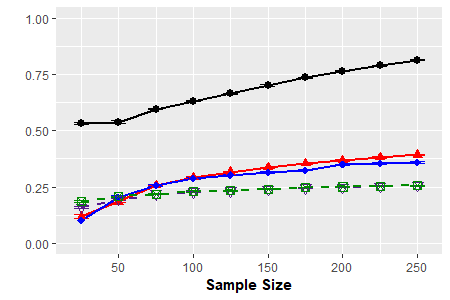}
		\caption{Prec:$p$=200,$d$=1}
	\end{subfigure}
	\begin{subfigure}[!htb]{.22\textwidth}
		\includegraphics[width=\textwidth, trim={0.5cm 0.2cm 0 0.2cm},clip]{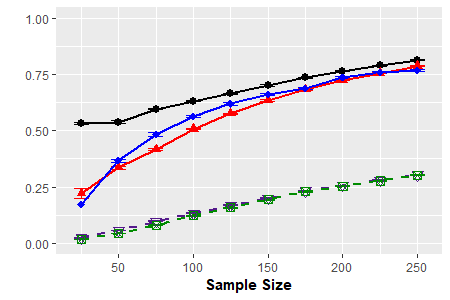}
		\caption{Reca:$p$=200,$d$=1}
	\end{subfigure}
	\begin{subfigure}[!htb]{.22\textwidth} 
		\includegraphics[width=\textwidth, trim={0.5cm 0.2cm 0 0.2cm},clip]{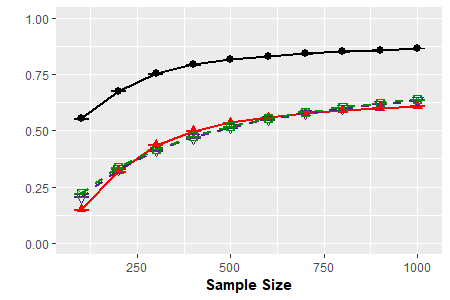}
		\caption{Prec:$p$=200,$d$=10}
	\end{subfigure}
	\begin{subfigure}[!htb]{.22\textwidth}
		\includegraphics[width=\textwidth, trim={0.5cm 0.2cm 0 0.2cm},clip]{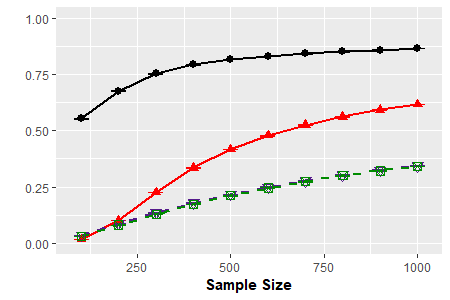}
		\caption{Reca:$p$=200,$d$=10}
	\end{subfigure}
	\begin{subfigure}[!htb]{.07\textwidth}
		\includegraphics[width=\textwidth, trim={14cm 0 0.3cm 1.5cm},clip]{plots/DAGLegend.png}
	\end{subfigure}
	%
	%\begin{subfigure}[!htb]{.22\textwidth} 
	%	\includegraphics[width=\textwidth, trim={0.5cm 0.2cm 0 0.2cm},clip]{plots/PoissonP500D1Precision.png}
	%	\caption{Precision:$p$=500,$d$=1}
	%\end{subfigure}
	%\begin{subfigure}[!htb]{.22\textwidth}
	%	\includegraphics[width=\textwidth, trim={0.5cm 0.2cm 0 0.2cm},clip]{plots/PoissonP500D1Recall.png}
	%	\caption{Recall:$p$=500,$d$=1}
	%\end{subfigure}
	%\begin{subfigure}[!htb]{.22\textwidth} 
	%	\includegraphics[width=\textwidth, trim={0.5cm 0.2cm 0 0.2cm},clip]{plots/PoissonP500D10Precision.png}
	%	\caption{Precision:$p$=500,$d$=10}
	%\end{subfigure}
	%\begin{subfigure}[!htb]{.22\textwidth}
	%	\includegraphics[width=\textwidth, trim={0.5cm 0.2cm 0 0.2cm},clip]{plots/PoissonP500D10Recall.png}
	%	\caption{Recall:$p$=500,$d$=10}
	%\end{subfigure}
	%\begin{subfigure}[!htb]{.07\textwidth}
	%	\includegraphics[width=\textwidth, trim={14cm 0 0.3cm 1.5cm},clip]{plots/DAGLegend.png}
	%\end{subfigure}
	\captionsetup{format=hang}
	\caption{Comparison of the MRS algorithm to the oracle, ODS, GES and MMHC algorithms in terms of precision and recall for Poisson SEMs with $p \in \{20, 200\}$ and $d \in \{1, 10\}$.}
	\label{fig:result001}
\end{figure*}

%%%% Interpretation 1 %%%%
As we can see in Fig.~\ref{fig:result001}, the MRS algorithm more accurately recovers the true directed edges as sample size increases. In addition, the MRS algorithm is more precise for small sparse graphs than for large-scale or dense graphs, given the same sample size. Hence it confirms that the MRS algorithm is consistent, and the sample bound $n = \Omega( d^{2} \log^9 p )$ depends on $p$ and $d$. 

%%%% Interpretation 2: GES, MMHC, faithfulness assumption %%%%
The MRS algorithm significantly outperforms state-of-the-art GES and MMHC algorithms in terms of both precision and recall, on average, except for cases $p =20, d=1, n \leq 50$. It is worth noting that the GES and MMHC algorithms are not consistent, because the recall for any tree graph must be zero in population, whereas the recall from GES and MMHC increases as sample size increases. Hence, we can conclude that the GES and MMHC algorithms find correct directed edges by finding incorrect v-structures. It is an expected result because the comparison methods only work with a non-faithful distribution, which rarely arises in finite sample settings~\citep{uhler2013geometry}.

%%%% Interpretation 3: ODS algorithm %%%%
Fig.~\ref{fig:result001} shows that the MRS and ODS algorithms have similar performance in identifying directed edges when the indegree is a small $d = 1$. It makes sense because the ODS algorithm recovers any Poisson DAG models if the moralized graph is sparse. In other words, the accuracy of the ODS algorithm may be poor for the non-sparse graph. Moreover, the ODS algorithm often fails to be implemented due to a lack of samples for the estimation of conditional variance, that is, $\sum_{i=1}^{n} \mathbf{1}( X_{S}^{(i)} = x ) < 2$ for all $x \in \mathcal{X}_{S}$. Table~\ref{Tbl:ODSresult} shows the number of failures in the ODS algorithm implementations for node size $p \in \{20, 50\}$ and sample size $n \in \{100, 200,..., 1000\}$ when the indegree is $d = 5$, and the degree of the moralized graph is at most $d_{m} = p-1$. It empirically confirms that the ODS algorithm requires a huge number of samples to be implemented when a true graph is not sparse. Hence, we do not apply the ODS algorithm for the graphs with $d = 10$.  It is consistent with our main result that our method can learn the Poisson SEMs with some hub nodes while the ODS algorithm might not. %However, we point out that ODS algorithm can learn Poisson DAG models with any form of link functions.

%%%% Number of successes of the ODS algorithm implementations %%%%
\begin{table*}
	\begin{center}
		\begin{tabular}{c c c c c c c c c c c} 
			n 					& 100 & 200 & 300 & 400 & 500 & 600 & 700 & 800 & 900 & 1000 \\ 	\hline
			p = 20			& 	199  & 175 & 107 & 64    & 1  & 0 & 0 & 0 & 0 &0 \\ 	\hline
			p = 50			&   200  & 200   &   200   &    199     &  192     &   179   &    151  &    140  &    99  &   86 \\ \hline
		\end{tabular}
		\captionsetup{format=hang}
		\caption{Number of failures in ODS algorithm implementations from among 200 sets of samples for different node sizes $p \in \{20, 50\}$, and sample sizes $n \in \{100, 200,..., 1000\}$, when the indegree is $d = 5$.\label{Tbl:ODSresult}  }
	\end{center} \vspace{-5mm}
\end{table*}  

\begin{figure*}[!t]
	\centering \hspace{-2mm}
	\begin{subfigure}[!htb]{.22\textwidth} 
		\includegraphics[width=\textwidth, trim={0.5cm 0.2cm 0 0.2cm},clip]{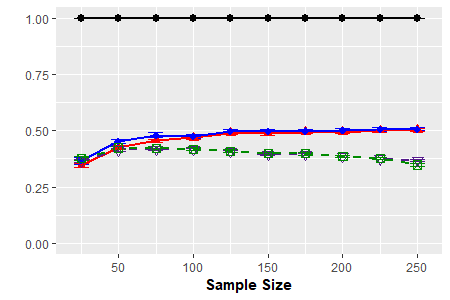}
		\caption{Prec:$p$=20,$d$=1}
	\end{subfigure}
	\begin{subfigure}[!htb]{.22\textwidth}
		\includegraphics[width=\textwidth, trim={0.5cm 0.2cm 0 0.2cm},clip]{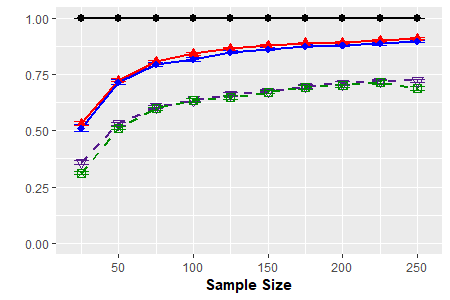}
		\caption{Reca:$p$=20,$d$=1}
	\end{subfigure}
	\begin{subfigure}[!htb]{.22\textwidth} 
		\includegraphics[width=\textwidth, trim={0.5cm 0.2cm 0 0.2cm},clip]{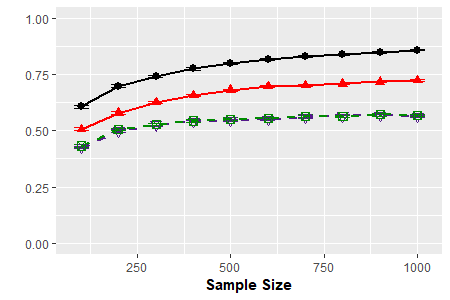}
		\caption{Prec:$p$=20,$d$=10}
	\end{subfigure}
	\begin{subfigure}[!htb]{.22\textwidth}
		\includegraphics[width=\textwidth, trim={0.5cm 0.2cm 0 0.2cm},clip]{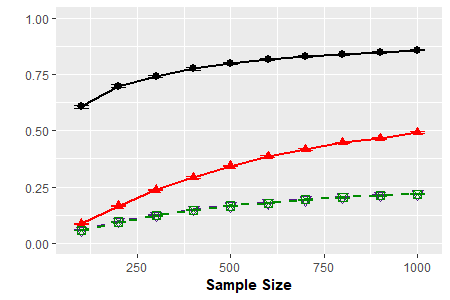}
		\caption{Reca:$p$=20,$d$=10}
	\end{subfigure}		
	\begin{subfigure}[!htb]{.07\textwidth}
		\includegraphics[width=\textwidth, trim={14cm 0 0.3cm 1.5cm},clip]{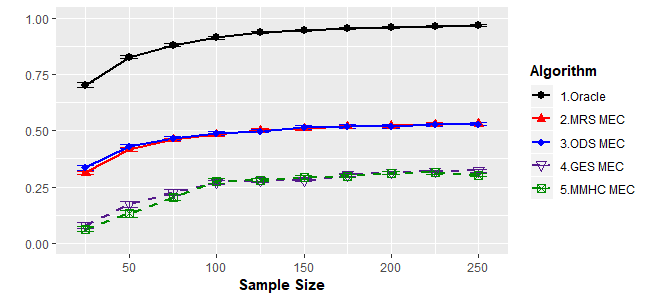}
	\end{subfigure}
	
	\begin{subfigure}[!htb]{.22\textwidth} 
		\includegraphics[width=\textwidth, trim={0.5cm 0.2cm 0 0.2cm},clip]{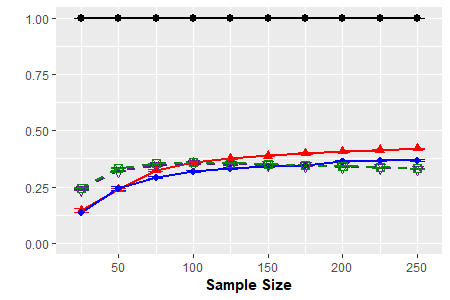}
		\caption{Prec:$p$=200,$d$=1}
	\end{subfigure}
	\begin{subfigure}[!htb]{.22\textwidth}
		\includegraphics[width=\textwidth, trim={0.5cm 0.2cm 0 0.2cm},clip]{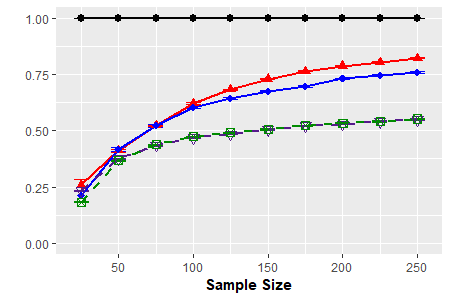}
		\caption{Reca:$p$=200,$d$=1}
	\end{subfigure}
	\begin{subfigure}[!htb]{.22\textwidth} 
		\includegraphics[width=\textwidth, trim={0.5cm 0.2cm 0 0.2cm},clip]{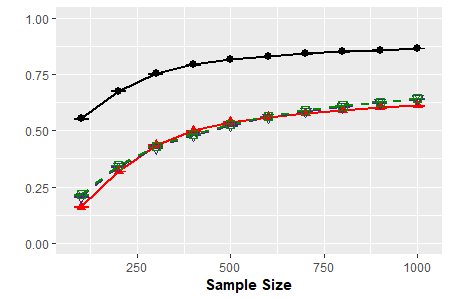}
		\caption{Prec:$p$=200,$d$=10}
	\end{subfigure}
	\begin{subfigure}[!htb]{.22\textwidth}
		\includegraphics[width=\textwidth, trim={0.5cm 0.2cm 0 0.2cm},clip]{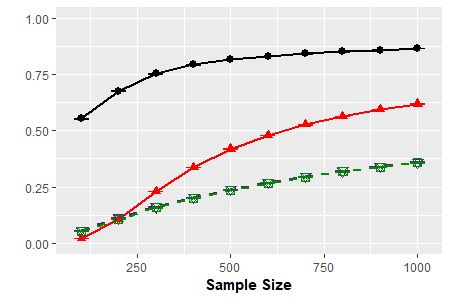}
		\caption{Reca:$p$=200,$d$=10}
	\end{subfigure}
	\begin{subfigure}[!htb]{.07\textwidth}
		\includegraphics[width=\textwidth, trim={14cm 0 0.3cm 1.5cm},clip]{plots/MECLegend.png}
	\end{subfigure}
	\captionsetup{format=hang}
	\caption{Comparison of the MRS algorithm to the oracle, ODS, GES, and MMHC algorithms in terms of the precision and recall for the MECs of Poisson SEMs with $p \in \{20,100\}$ and $d \in \{1, 10\}$.}
	\label{fig:result002}
\end{figure*}

%%%%
Fig.~\ref{fig:result002} shows the analogous results for the recovery of MECs, in which the MRS and all comparison algorithms consistently  learn the true MECs. The performance of the MRS algorithm gets better as sample size increases or node size decreases. In addition, we can see that the MRS algorithm still recovers the MEC of the Poisson SEM better on average than the comparison methods. However, it must be pointed out that our MRS algorithm applies to Poisson SEMs~\eqref{eq:PDAGGLM}, while the ODS algorithm accurately learns sparse Poisson DAG models where arbitrary link functions are allowed. In addition, the GES and MMHC algorithms apply to more general classes of DAG models.

\subsection{Random Poisson DAG Models}

\label{SecRanPDAG}

\begin{figure*}[!t]
	\centering \hspace{-2mm}
	\begin{subfigure}[!htb]{.22\textwidth} 
		\includegraphics[width=\textwidth, trim={0.5cm 0.2cm 0 0.2cm},clip]{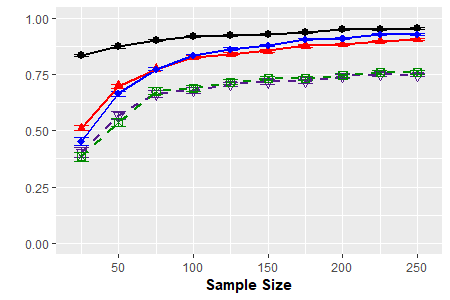}
		\caption{Prec:$p$=20,$d$=2}
	\end{subfigure}
	\begin{subfigure}[!htb]{.22\textwidth}
		\includegraphics[width=\textwidth, trim={0.5cm 0.2cm 0 0.2cm},clip]{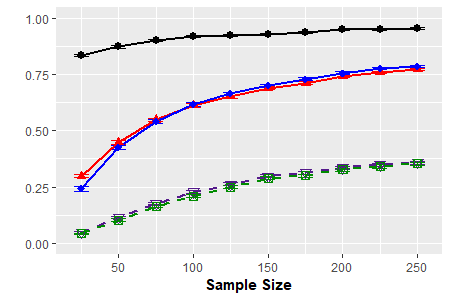}
		\caption{Reca:$p$=20,$d$=2}
	\end{subfigure}
	\begin{subfigure}[!htb]{.22\textwidth} 
		\includegraphics[width=\textwidth, trim={0.5cm 0.2cm 0 0.2cm},clip]{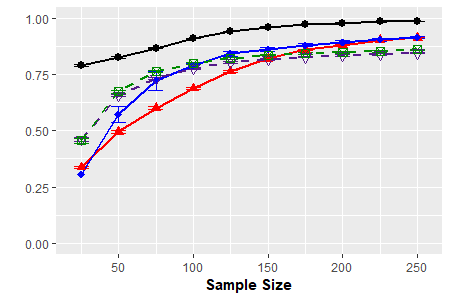}
		\caption{Prec:$p$=100,$d$=2}
	\end{subfigure}
	\begin{subfigure}[!htb]{.22\textwidth}
		\includegraphics[width=\textwidth, trim={0.5cm 0.2cm 0 0.2cm},clip]{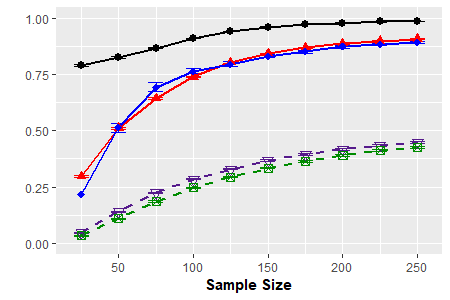}
		\caption{Reca:$p$=100,$d$=2}
	\end{subfigure}
	\begin{subfigure}[!htb]{.07\textwidth}
		\includegraphics[width=\textwidth, trim={14cm 0 0.3cm 1.5cm},clip]{plots/DAGLegend.png}
	\end{subfigure}
	\captionsetup{format=hang}
	\caption{Comparison of the MRS algorithm to the oracle, ODS, GES and MMHC algorithms in terms of the precision and recall for Poisson DAG models with $p \in \{20,100\}$, $d = 2$, and the identity link function.}
	\label{fig:result003} 
\end{figure*}

%%%%% Introduction %%%%%
When the data are generated by a random Poisson DAG model~\eqref{eq:CondPDAGM} where $g_j$ is not the standard log link function, our MRS algorithm is not guaranteed to estimate the true directed acyclic graph and its ordering. Hence, an important question is how sensitive our method is to deviations from the link assumption. In this section, we empirically investigate this question. 

%%%%% Settings %%%%%%
We generated the $200$ samples with the same procedure specified in Section~\ref{SecRanPSEM}, but with the indegree constraint $d = 2$, and except that identity link function $g_j( \eta ) = \eta$ and the range of parameters was $\theta_{jk} \in [-1.5, -0.5] \cup [0.5, 1.5]$. We note that the link function must be positive, but we allow the negative value of $\theta_{jk}$ by randomly choosing $\theta_{j} \in [1, 10]$. If any Poisson rate parameter is negative, we regenerated the parameters.

%%%%% Comparison Method GES, MMHC: Interpretation 1%%%%%
In Fig.~\ref{fig:result003}, we compare the MRS to state-of-the-art ODS, GES and MMHC algorithms for varying sample size $n \in \{25,50,...,250\}$, and node size $p \in \{20, 100\}$. Fig.~\ref{fig:result003} shows that the MRS algorithm consistently recovers the true graph, and hence, we can see that the MRS algorithm is not so sensitive to deviations from the link assumption. Comparing it to the ODS algorithm, the MRS algorithm shows slightly worse performance because the ODS algorithm is designed to learn general Poisson DAG models with any type of link functions. However, we can see that the MRS algorithm still performs better than the greedy search-based methods in both average precision and recall. 

\subsection{Random Poisson and Truncated Poisson Markov Random Fields}

\begin{figure*}[!t]
	\centering \hspace{-2mm}
	\begin{subfigure}[!htb]{.22\textwidth} 
		\includegraphics[width=\textwidth, trim={0.5cm 0.2cm 0 0.2cm},clip]{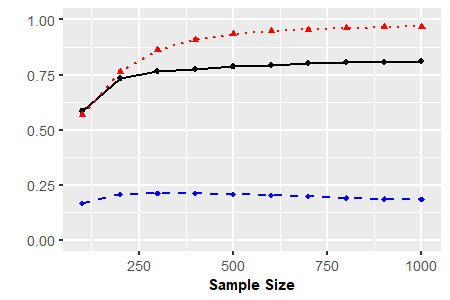}
		\caption{Poisson: Prec}
	\end{subfigure}
	\begin{subfigure}[!htb]{.22\textwidth}
		\includegraphics[width=\textwidth, trim={0.5cm 0.2cm 0 0.2cm},clip]{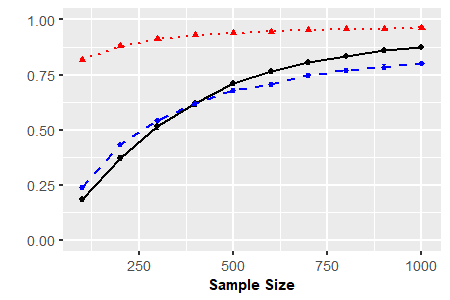}
		\caption{Poisson: Reca}
	\end{subfigure}
	\begin{subfigure}[!htb]{.22\textwidth} 
		\includegraphics[width=\textwidth, trim={0.5cm 0.2cm 0 0.2cm},clip]{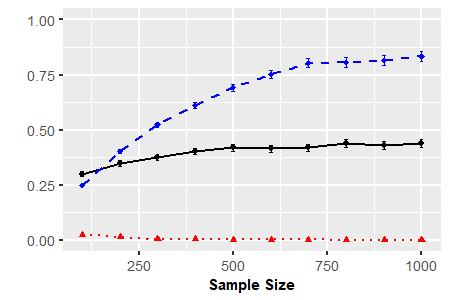}
		\caption{Truncated: Prec}
	\end{subfigure}
	\begin{subfigure}[!htb]{.22\textwidth}
		\includegraphics[width=\textwidth, trim={0.5cm 0.2cm 0 0.2cm},clip]{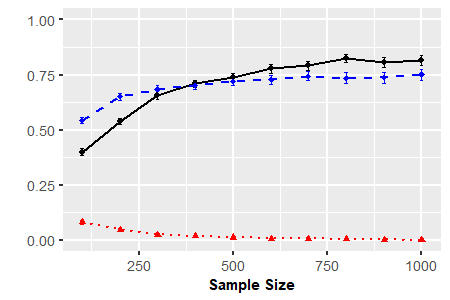}
		\caption{Truncated: Reca}
	\end{subfigure}
	\begin{subfigure}[!htb]{.07\textwidth}
	\includegraphics[width=\textwidth, trim={20.2cm 0 0.3cm 1.5cm},clip]{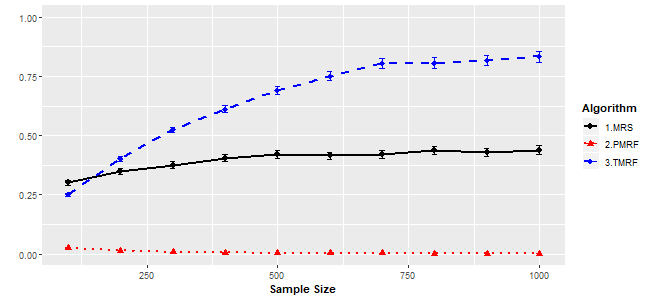}
	\end{subfigure}
	\captionsetup{format=hang}
	\caption{Comparison of the MRS algorithm to the Poisson MRF learning (PMRF) and truncated Poisson MRF learning (TMRF) algorithms in terms of the precision and recall for undirected edges of random 20-nodes Poisson MRFs and truncated Poisson MRFs with $d_m = 5$, and $R = 100$.}
	\label{fig:result005} 
\end{figure*}

%%%%% Introduction %%%%%
When samples are generated by a Poisson or truncated Poisson MRF, our MRS algorithm is not guaranteed to find the true dependence relationships of variables. Hence, it is also important to investigate how well our algorithm recovers undirected edges when multivariate count data is from an MRF. In this section, we compare our MRS algorithm to state-of-the-art Poisson MRF (PMRF) and truncated Poisson MRF learning (TMRF) algorithms~\citep{yang2013poisson,yang2015graphical,inouye2017review} when multivariate count data is from Poisson MRFs and truncated Poisson MRFs, respectively. We used the R package XMRF~\citep{wan2016xmrf} for truncated Poisson MRFs.

%%%% Simulation Settings %%%%
We generated $100$ samples of 20-nodes random Poisson MRF and truncated Poisson MRF with the randomly generated underlying undirected graphs, respectively. For Poisson MRFs, we set the maximum Markov blanket $d_m = 5$ and the non-zero parameters in Equation~\eqref{eq:PoissonMRF} was generated uniformly at random in the range $\theta_{j} \in [1, 2]$, but we fixed $\theta_{jk} = -0.1$ for all $j \in V$. This is a similar setting used in \cite{yang2015graphical}. For truncated Poisson MRFs, we set $d_m = 5$, $\theta_{j} = 0$, $\theta_{jk} = 0.1$, and the truncation level is $R = 100$, meaning that all samples are less than 100 (see details in Equation 3 of \citealp{yang2013poisson}). In terms of the choice of regularization parameters for the MRS and PMRF algorithms, we used five-fold cross validation as we used in Section~\ref{SecRanPSEM}. For the TMRF algorithm, we set  the regularization parameters to 0.1 since this value seems to work well.

%%%%% Comparison Method PMRF: Interpretation 1%%%%%
Fig.~\ref{fig:result005} compares the MRS algorithm to state-of-the-art PMRF and TMRF algorithms in terms of recovering undirected edges by varying sample size $n \in \{100, 200,..., 1000\}$. For a fair comparison, we used the skeleton of the estimated MEC via the MRS algorithm, because our algorithm returns a DAG. As we can see in Fig.~\ref{fig:result005}, the MRS algorithm consistently finds the true edges from both Poisson MRF and truncated Poisson MRF samples. Hence, we empirically verify that the MRS algorithm can recover some dependence relationships of variables even if samples are from Poisson or truncated Poisson MRFs.

%%%%% Comparison Method PMRF: Interpretation 1%%%%%
Fig.~\ref{fig:result005} also shows that the MRS algorithm performs significantly worse than the comparison PMRF and TMRF algorithm, on average, when samples are from Poisson MRFs and truncated Poisson MRFs, respectively. It is an expected result because the PMRF and TMRF algorithms are for learning Poisson MRFs and truncated MRFs, while our algorithm is for Poisson SEMs. However, it is worth noting that the TMRF algorithm seems not to work on average when samples are from a Poisson MRF in our setting. It is mainly because the TMRF algorithm is for learning truncated Poisson MRFs, not Poisson MRFs. We emphasize that, in another setting where $\theta_j$ is fixed to 1, the TMRF algorithm works much better. It is also worth noting that the PMRF algorithm seems not to recover any undirected edges when samples are from a truncated Poisson MRF. It can be clearly explained by the fact that the PMRF algorithm cannot capture the positive dependencies, however all parameters are positive in our setting. 

\subsection{Computational Complexity}

\begin{figure*}[!t]
	\centering
	\begin{subfigure}[!htb]{.30\textwidth} 
		\includegraphics[width=\textwidth, height = 2.5cm, trim={0.0cm 0.2cm 0 0.2cm},clip]{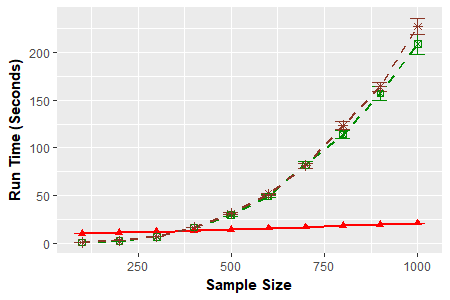}
		\caption{$p=100, d=5$}
	\end{subfigure}
	\begin{subfigure}[!htb]{.30\textwidth} 
		\includegraphics[width=\textwidth, height = 2.5cm, trim={0.0cm 0.2cm 0.2cm 0.2cm},clip]{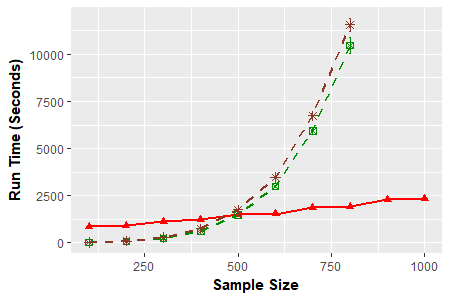}
		\caption{$p=500, d=5$}
	\end{subfigure}
	\begin{subfigure}[!htb]{.30\textwidth} 
		\includegraphics[width=\textwidth, height = 2.5cm, trim={0.0cm 0.2cm 0.2cm 0.2cm},clip]{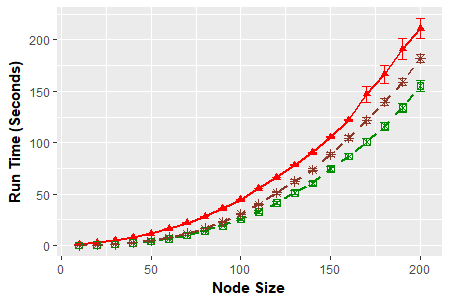}
		\caption{$n=500, d=5$}
	\end{subfigure}	
	\begin{subfigure}[!htb]{.07\textwidth}
		\includegraphics[width=\textwidth, trim={9.5cm 0 0.3cm 1.5cm},clip]{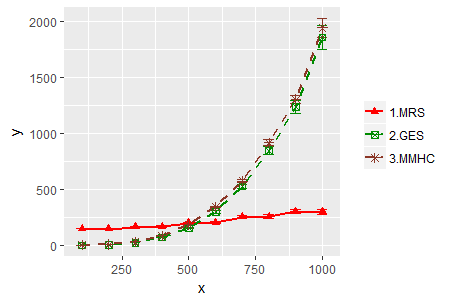}
	\end{subfigure}
	\captionsetup{format=hang}
	\caption{Comparison of the MRS algorithm to the GES and MMHC algorithms in terms of the running time with respect to node size $p$ and sample size $n$}
	\label{fig:result004} 
\end{figure*} 

Fig.~\ref{fig:result004} compares the run-time of the MRS, GES, and MMHC algorithms for learning Poisson SEMs with indegree $d = 5$ by varying sample size $n \in \{100, 200, ... ,1000\}$ with fixed node size $p \in \{100, 500\}$, and varying node size $p \in \{10, 20, ..., 200\}$ with fixed sample size $n = 500$. Fig.~\ref{fig:result004} supports the worst case computational complexity $O(np^3)$ discussed in Section~\ref{SecComp}. In addition, it shows that the MRS algorithm is significantly faster than the greedy search-based GES and MMHC algorithms when a sample size is large ($n > 500$). 
%The greedy search-based comparison algorithms are ironically not suitable for large-scale graphs with large samples due to the huge computational cost. For $p = 500$ and $n = 1000$, the expected run time of the algorithms is about half a day.

\section{Real Multivariate Count Data: MLB Statistics}

\label{SecReal}

% introduction of this section
We now apply the MRS algorithm and state-of-the-art ODS and MMHC algorithms to a simple data set that involves multivariate count data that models baseball statistics for Major League Baseball (MLB) players during the 2003 season. To the best of our knowledge, our MRS algorithm is the only algorithm that provides a reliable and scalable approach to non-sparse DAG learning with multivariate count data although it is under strong assumptions. In particular, other approaches, such as PC, MMHC, and approaches based on conditional independence testing, suffer severely from the fact that we are dealing with count variables where the number of discrete states is potentially infinite. In addition, ODS algorithm cannot deal with a non-sparse graph such as a graph containing hub nodes. Lastly, both Poisson MRF and truncated Poisson MRF may provide an extremely complicated graph because it connects all pairs of nodes having a common child like a moralized graph.

%In terms of real data applications, one of the advantages of our MRS algorithm is that it provides a scalable approach for learning DAG models when variables are counts. 

Our original data set consists of 800 MLB player salary and batting statistics from the 2003 season (see R package Lahman in \citealp{lahmanRpackage} for detailed information). The data set contains 23 covariates: Salary, Number of: Games Played (G), At Bats (AB), Runs (R), Hits (H), Doubles (X2B), Triples (X3B), Home Runs (HR), Runs Batted In (RBI), Stolen Bases (SB), times Caught Stealing (CS), Bases on Balls (BB), Strikeouts (SO), Intentional Walks (IBB), times Hit by Pitch (HBP), Sacrifice Hits (SH), Sacrifice Flies (SF), and times Grounded into Double Plays (GIDP), plus Player ID, Year ID, Stint, Team ID, and League ID. However, we eliminated Player ID, Year ID, Stint, Team ID, and League ID because our focus is to find the directional or causal relationships between salary and batting statistics. In addition, we only considered players in the top 25\% in terms of the number of games played, because the baseball statistics relationships from players who played only a few games could be uncertain. Therefore, the data set we considered contained 18 variables and 200 observations. 

We assumed each node to a conditional distribution given its parents is Poisson because most MLB statistics, except for salary, reflect the number of successes or attempts that were counted during the season. Hence, we applied the MRS algorithm for Poisson DAG models with leave-one-out cross validation to choose the tuning parameters, and we chose the largest value where mean squared error is within 2.5 standard error of the minimum mean squared error, because we prefer a sparse graph containing only legitimate edges.
\begin{figure}
	\centering
	\includegraphics[width=0.60\linewidth, trim={6cm 11cm 6cm 10cm},clip]{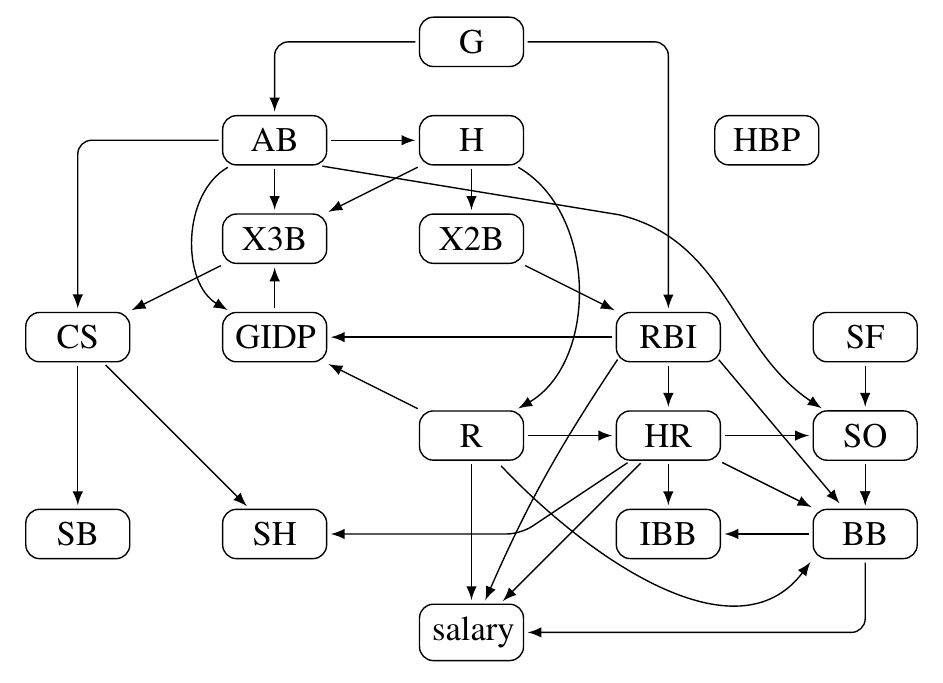}
	\captionsetup{format=hang}
	\caption{MLB player statistics directed graph estimated by the MRS algorithm for Poisson DAG models.}
	\label{fig:MLBDAG}
\end{figure}

Fig.~\ref{fig:MLBDAG} shows the directed graph estimated by our MRS algorithm. The estimated graph reveals clear causal/directional relationships between batting statistics. This makes sense, because players with larger numbers of HR, BB, RBI, and/or R have a better salary. The more games played, or the more batting chances, the higher H, BB, SO, RBI, and other statistics. Moreover, the higher the total number of hits, the more X2Bs, X3Bs, Rs and the fewer SOs. Players with more home runs and base on balls get intentional walks more frequently. Lastly, the more stolen bases are attempted, the more they are caught stealing, because there is no success without failure.

We acknowledge that our proposed DAG model returns many errors due to restrictive assumptions that are not completely satisfied by the real data. However, the benefit is best seen by comparing MRS to other DAG learning approaches and an undirected graphical model for multivariate count data. In particular, we applied Poisson undirected graphical models~\citep{yang2015graphical} in which $\ell_1$-regularized Poisson regressions are applied. We provide the estimated undirected graph with the largest tuning parameter where mean squared of error is within 2.5 standard error of the minimum mean squared error. The estimated undirected graph in Fig.~\ref{fig:MLBDAG2} (left side) shows that a lot of nodes are connected by edges, that many edges are unexplainable, and that some legitimate edges are missing (e.g., [H, X3B], [SB, CS] are not connected), because the Poisson undirected graphical model only permits negative conditional relationships, whereas most variables are positively correlated. Hence, it may not be useful to understand the relationships between MLB statistics. 

We also compared the MMHC algorithm. As discussed, the MMHC algorithm does not guarantee to find a complete directed graph, and prefers a sparser graph when the faithfulness assumption is violated, which often arises in finite sample settings~\citep{uhler2013geometry}. Hence, the estimated directed graph in Fig.~\ref{fig:MLBDAG2} (right side) is extremely sparse, with only four directed edges: [H, HR], [SO, HR], [HR, RBI], and [SF, RBI]. Lastly, ODS algorithm failed to be implemented as expected because of some hub nodes such as the number of games, at bats, and runs batted in. 

\begin{figure}
	\begin{center}
		\includegraphics[width=0.48\linewidth, trim={5.8cm 11cm 6cm 10cm},clip]{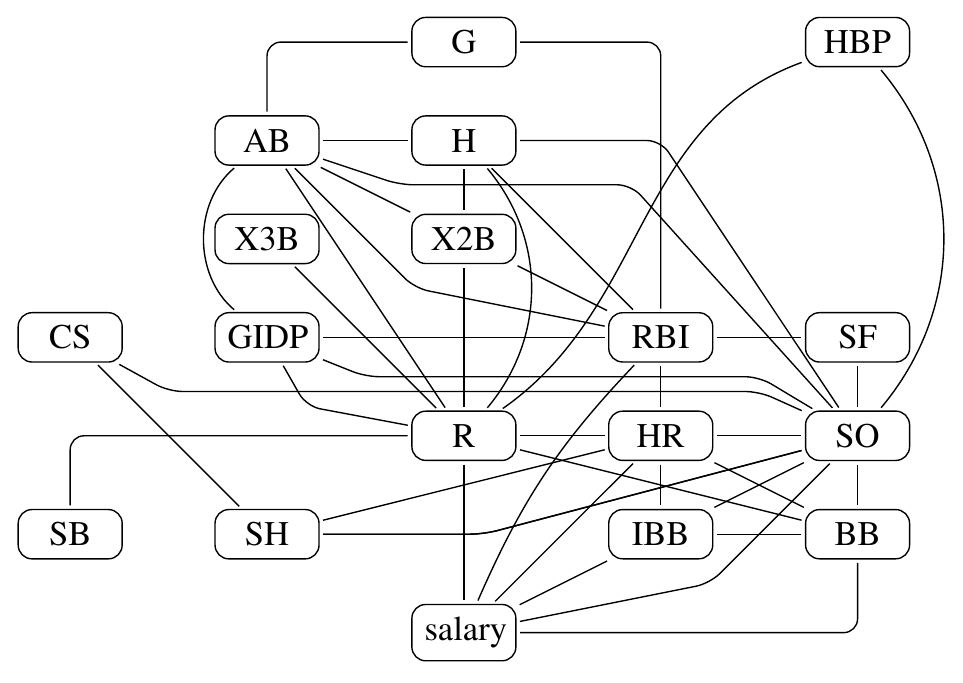}
		\includegraphics[width=0.48\linewidth, trim={5.8cm 11cm 6cm 10cm},clip]{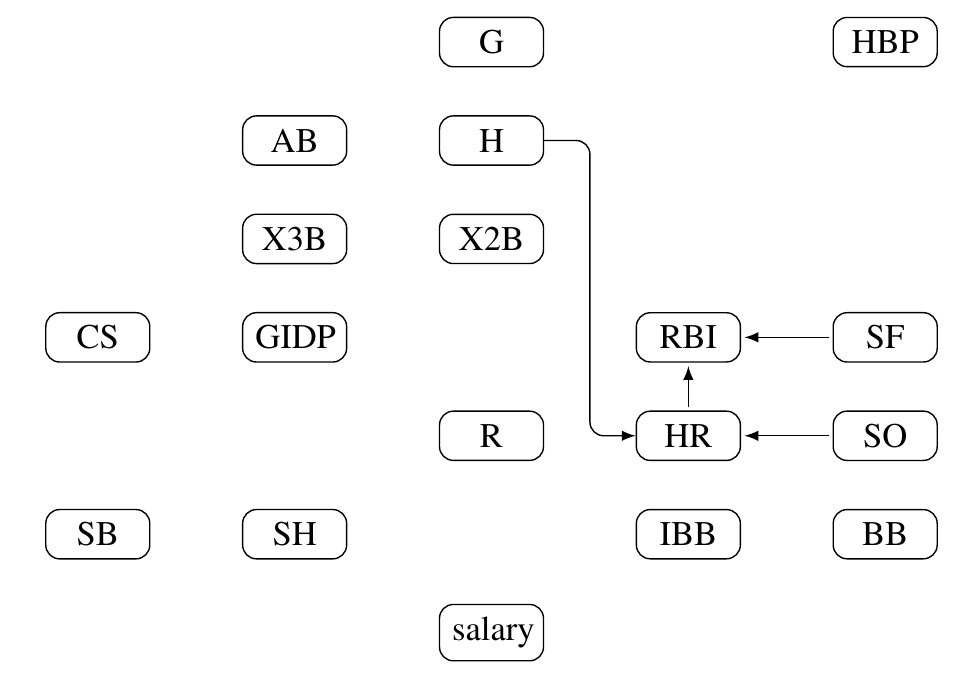}
		\captionsetup{format=hang}
		\caption{MLB player statistics undirected graph estimated by $\ell_1$-penalized likelihood regression (left) and a directed acyclic graph estimated by the MMHC algorithm (right).}
		\label{fig:MLBDAG2}
	\end{center}	
\end{figure}

Since our method is the first identifiability result for the strongly correlated count data when variables are directional/causal relationships and there exist hub variables, to the best of our knowledge, our method better identifies the directional/causal relationships between MLB statistics. However, we acknowledge that, like most other DAG-learning approaches, very strong assumptions, such as dependency, incoherence, are required for reliable recovery. 

\section{Future Works}

\label{SecFuture}

%In this paper, we proved that the non-zero parameter assumption guarantees for identifying Poisson SEMs, that is, the underlying graph can be recovered from the joint distribution. We also proposed the MRS algorithm for learning Poisson SEMs, and the algorithm requires relatively acceptable assumptions and better sample complexity than the Poisson DAG learning algorithm in \citet{park2017learning}. We point out that our MRS algorithm does not require faithfulness, but allows some hub nodes in a graph. However, we emphasize that our better sample complexity and assumptions of the MRS algorithm can be achieved by not only exploiting a new identifiability condition, but the standard link function for the dependencies. We tested our MRS algorithm on various simulated data sets. The experiments support our theoretical results, empirically confirms that the MRS algorithm can capture dependency of variables in Poisson DAG models and Poisson MRFs. 

Several topics remain for future works. Although our assumptions are similar to the assumptions in the previous works of $\ell_1$-regularized Poisson regression, our assumptions could be very restrictive. In addition, they cannot be confirmed from data. However, we conjecture that the assumptions are satisfied with a high probability under mild conditions, and one may be able to prove this. In addition, it is an important problem of finding the minimax rate of the Poisson DAG models, and it should be investigated in the future. Lastly, it would be also interesting to explore if our idea can be applied to other structural equation models with Binomial, Negative Binomial, Exponential, and Gamma distributions. We believe that our node-wise $\ell_1$-regularized based approach can be extended to the identifiable linear SEMs under some suitable conditions.

\bibliographystyle{IEEEtran}
\bibliography{PSEM_reference}

\newpage

\appendix

\section{Proof for Proposition~\ref{prop:moments ratio} }

\label{SecPropProof1}

\begin{proof}
	For a notational simplicity, we define a moments related function for Poisson, $f(\mu) = \mu + \mu^2$ for $\mu >0$. Then, for any node $j \in V$, any non-empty set $S_j \subset \nd(j)$, 
	\begin{eqnarray*}
		\E( X_j^2 \mid S_j ) &=& \E( \E(X_j^2 \mid X_{\pa(j)} ) \mid S_j ) = \E( f( \E(X_j \mid X_{\pa(j)} ) ) \mid S_j ). 
	\end{eqnarray*}
	
	Using the Jensen’s inequality and $f(\cdot)$ is convex, we have,  
	\begin{equation*}
	\E( f( \E(X_j \mid X_{\pa(j)} ) ) \mid S_j ) \geq f( \E( \E(X_j \mid X_{\pa(j)} ) \mid S_j  ) ) =  f( \E( X_j  \mid S_j) ). 
	\end{equation*}
	
	Using the fact that $\E(X_j \mid X_{\pa(j)} ) = g_j(X_{\pa(j)})$ and it is non-degenerated by definition, the equality only holds when $S_j$ contains all parents of $j$, $\pa(j) \subset S_j \subset \nd(j)$.
	
	By restating the above inequality, we have, 
	$$
	\E( X_j^2 \mid S_j ) - \E( X_j  \mid S_j) -   \E( X_j  \mid S_j)^2  \geq 0.
	$$
	
	In addition, by taking the expectations, we have,
	$$
	\E( X_j^2 ) - \E\left(   \E( X_j \mid X_{S_j}) + \E( X_j \mid X_{S_j})^2 \right) \geq 0.
	$$
	
	Since $j$ and $S_j$ are arbitrary, we complete the first part of the proof. 
	
	Now, we prove that $ \E( X_j^2) \geq  \E\left(   \E( X_j \mid X_{S_j}) + \E( X_j \mid X_{S_j})^2 \right) $ is equivalent to $ \E( \var( \E(X_j \mid \pa(j) ) \mid X_{S_j} ) ) \geq 0$. 
	Using the total variance decomposition, we have, 
	\begin{eqnarray*}
		\E( \var(X_j \mid X_{S_j} ) ) = \E( \E( \var(X_j \mid X_{\pa(j)} ) \mid X_{S_j}) ) + \E( \var( \E( X_j \mid X_{\pa(j)} ) \mid X_{S_j}) ). 
	\end{eqnarray*}
	
	Using the fact that the conditional distribution, $X_j \mid X_{\pa(j)}$, is Poisson where its mean and variance are equal, we have, 
	$$
	\E( \var( \E( X_j \mid X_{\pa(j)} ) \mid X_{S_j}) ) = \E( \var(X_j \mid X_{S_j} ) ) - \E(X_j).
	$$
	
	Using the definition of the conditional variance, we have,
	$$
	\E( \var(X_j \mid X_{S_j}) ) - \E(X_j) = \E( X_j^2 ) - \E\left(   \E( X_j \mid X_{S_j}) + \E( X_j \mid X_{S_j})^2 \right). 
	$$ 
	
	Therefore, we complete the proof. 
\end{proof}

\section{Proof for Theorem~\ref{ThmMainIdentifiability} }

\label{SecMainThmProof0}

\begin{proof}
	Without loss of generality, we assume the true ordering is unique, and $\pi = (\pi_1,...,\pi_p)$. For simplicity, we define $X_{1:j} = (X_{\pi_1},X_{\pi_2},\cdots,X_{\pi_j})$ and $X_{1:0} = \emptyset$. In addition, we define a moments related function, $f(\mu) = \mu + \mu^2$. 
	
	We now prove identifiability of Poisson DAG models using mathematical induction: 
	
	\textbf{Step (1)} For the first step $\pi_1$, using Proposition~\ref{prop:moments ratio}, we have 
	$\E( X_{\pi_1}^2 ) =  \E( f(\E( X_{\pi_1})) )$, while for any node $j \in V \setminus \{\pi_1\}$: $\E( X_j^2 ) > \E( f( \E( X_j) ) ).$
	
	Hence, we can determine $\pi_1$ as the first element of the causal ordering.
	
	\textbf{Step (m-1)} For the $(m-1)^{th}$ element of the ordering, assume that the first $m-1$ elements of the ordering and their parents are correctly estimated. 
	
	\textbf{Step (m)} Now, we consider the $m^{th}$ element of the causal ordering and its parents. It is clear that $\pi_{m}$ achieves $\E( X_{\pi_m}^2) = \E( f(\E( X_{\pi_m} \mid X_{1:(m-1)})) )$. However, for $j \in \{ \pi_{m+1},\cdots, \pi_{p} \}$, $\E( X_j^2 ) >  \E( f( \E( X_j \mid X_{1:(m-1)}) ) )$ by Proposition~\ref{prop:moments ratio}. Hence, we can estimate a true $m^{th}$ component of the ordering $\pi_m$. 
	
	In terms of the parent search, it is clear that by conditional independence relations naturally encoded by factorization~\eqref{eq:factorization}
	$
	\E( X_{\pi_m}^2) = \E( f(\E( X_{\pi_m} \mid X_{1:(m-1)})) ) =  \E( f(\E( X_{\pi_m} \mid X_{\pa(\pi_m)})) ).
	$ 
	Hence, we can also choose the minimum conditioning set from among $X_{1:(m-1)}$ as the parents of $\pi_m$ such that the above moments relation holds. By mathematical induction, this completes the proof. 
	
\end{proof}

\section{Proof for Theorem~\ref{ThmMainTheorm}: Parents Recovery}

\label{SecThmMainProof1}

\begin{proof}
%%% Introduction %%%
We provide the proof for Theorem~\ref{ThmMainTheorm} using the \emph{primal-dual witness method} that is also used many other works \citep{ meinshausen2006high, wainwright2006high, ravikumar2011high, yang2015graphical}. In this proof, we show in Appendix~\ref{SecThmMainProof1}, the error probability for the recovery of the parents of a node $\pi_j$ from among all the nodes given the partial ordering $(\pi_{1},\pi_{2},...,\pi_{j-1})$ via $\ell_1$-regularized regression. In Appendix~\ref{SecMainThmProof2}, the error bounds for the recovery of the ordering both via $\ell_1$-regularized regression. 

%%% Notations %%%

%This notation often makes it unnecessary to use the argument b in the term pa(b)i, and we use the simpler expression where possible. 
%We use the notation of conditional variables rather than conditional distributions
%To make the arguments easier to understand, we introduce the following notation 

Without loss of generality, let the true ordering be $\pi = (1,2,...,p)$, and hence, $\pi_{1:j} = (\pi_1, \pi_2, ..., \pi_j) = (1,2,...,j)$. For ease of notation, $[\cdot]_{k}$ and $[\cdot]_{S}$ denote parameters corresponding to variable $X_k$ and random vector $X_S$, respectively. In order to make the arguments easier to understand, we restate the negative log likelihood~\eqref{eq:likelihood} and related arguments. 

First, we define a new parameter vector $\theta_{S_j} \in \mathbb{R}^{|S_j|}$ without parameter $\theta_j$ corresponding to the node $j$ since the node $j$ is not penalized in regression problem~\eqref{eqn:theta_jS2}. Then, the conditional negative log-likelihood of the GLM for $X_j$ given $X_{S_j}$ can be written as: 
\begin{equation}
\label{l:ThetaSj}
\ell_j^{S_j}( \theta_{S_j}; X^{1:n} ) := \frac{1}{n} \sum_{i = 1}^{n} \left( -X_j^{(i)} \langle \theta_{S_j}, X_{S_j}^{(i)} \rangle + \exp\big( \langle \theta_{S_j}, X_{S_j}^{(i)} \rangle \big)  \right),
\end{equation}
where $\langle \cdot, \cdot \rangle$ is an inner product. 

We also define $\theta_{S_j}^* \in \mathbb{R}^{|S_j|}$ for Equation~\eqref{ThetaSj}:
\begin{equation}
\label{eq:ThetaS*}
\theta_{S_j}^* := \arg \min_{\theta \in \mathbb{R}^{|S_j|}} \E \left( - X_j( \langle \theta, X_{S_j} \rangle ) + \exp( \langle \theta, X_{S_j} \rangle )  \right).
\end{equation}

We define a set non-zero elements index of $\theta_{S_j}^{*}$ as in Equation~\eqref{eq:subparents}, 
$
T_j := \{k \in S_j \mid [\theta_{S_j}^{*}]_{k} \neq 0 \}
$
where $\theta_{S_j}^{*}$ is in Equation \eqref{eq:ThetaS*}.

The main goal of the proof is to find the unique minimizer of the following convex problem:
\begin{equation}
\label{eq:Objective}
\widehat{\theta}_{S_j} := \arg \min_{\theta \in \mathbb{R}^{ |S_j|  }  } \mathcal{L}_j( \theta, \lambda_j) = \arg \min_{\theta \in \mathbb{R}^{|S_j| }  } \{ \ell_j^{S_j} (\theta ; X^{1:n}) + \lambda_j \| \theta \|_1 \}.
\end{equation}

By setting the \emph{sub-differential} to $0$, $\widehat{\theta}_{S_j}$ satisfies the following condition: 
\begin{equation}
\label{eq:Contraint1}
\bigtriangledown_\theta \mathcal{L}_j^{S_j}( \widehat{\theta}_{S_j}, \lambda_j ) = \bigtriangledown_\theta \ell_j^{S_j}( \widehat{\theta}_{S_j} ; X^{1:n}) + \lambda_j \widehat{Z}_j^{S_j}  = 0
\end{equation}
where $\widehat{Z}_j^{S_j} \in \mathbb{R}^{|S_j|}$ and $[ \widehat{Z}_j^{S_j} ]_t = \mbox{sign}([\widehat{\theta}_{S_j}]_{t})$ if  $t \in T_j$, otherwise $[ \widehat{Z}_j^{S_j} ]_t < 1$.  

Lemma~\ref{lemma:uniqueness} directly follows from the prior work \citep{yang2015graphical}, where each node's conditional distribution is in the form of a generalized linear model. 

\begin{lemma}[Uniqueness of Solution, Lemma 8 in \citealp{yang2015graphical}]
	\label{lemma:uniqueness}
	Suppose that \\$| [ \widehat{Z}_j^{S_j} ]_t | < 1$ for $t \notin T_j$ in Equation~\eqref{eq:Contraint1}. Then, the solution $\widehat{\theta}_{S_j}$ of Equation ~\eqref{eq:Objective} satisfies $[\widehat{\theta}_{S_j}]_{t} = 0$ for all $t \notin T_j$. Furthermore, if the sub-matrix of Hessian matrix $Q_{T_j T_j}^{S_j}$ is invertible, then $\widehat{\theta}_{S_j}$ is unique. 
\end{lemma}

The remainder of the proof is to show $| [ \widehat{Z}_j^{S_j} ]_t | < 1$ for all $t \notin T_j$. Note that the restricted solution in Equation~\eqref{eq:ResObjective} is $(\widetilde{\theta}_{S_j}, \widetilde{Z}_j^{S_j})$ and the unrestricted solution in Equation~\eqref{eq:Objective} is $(\widehat{\theta}_{S_j}, \widehat{Z}_j^{S_j})$. Equation~\eqref{eq:Contraint1} with the dual solution can be represented by 
\begin{equation}
\label{eq:Contraint2}
\bigtriangledown^2 \ell_j^{S_j}( \theta_{S_j}^*;X^{1:n})( \widetilde{\theta}_{S_j} - \theta_{S_j}^* ) = -\lambda_j \widetilde{Z}_j^{S_j} - W_j^{S_j} + R_j^{S_j}
\end{equation}
where
\begin{itemize}
	\item[(a)] $W_{j}^{S_j}$ is the sample score function:
	\begin{equation}
	\label{eq:Wn}
	W_{j}^{S_j} := - \bigtriangledown \ell_j(\theta_{S_j}^*;X^{1:n}).
	\end{equation}
	\item[(b)] $R_{j}^{S_j} = (R_{jk}^{S_j})_{k \in S_j}$ and $R_{jk}^{S_j}$ is the remainder term by applying the coordinate-wise mean value theorem: 
	\begin{equation}
	\label{eq:Rn}
	R_{jk}^{S_j} := [ \bigtriangledown^2 \ell_j^{S_j}(\theta_{S_j}^*;X^{1:n}) - \bigtriangledown^2 \ell_j^{S_j}(\bar{\theta}_{S_j} ;X^{1:n})]_k^T (\widetilde{\theta}_{S_j} - \theta_{S_j}^*).
	\end{equation}
	Here $\bar{\theta}_{S_j}$ is a vector on the line between $\widetilde{\theta}_{S_j}$ and $\theta_{S_j}^*$, and $[\cdot]_k^T$ is the row of a matrix corresponding to variable $X_k$.
\end{itemize}

Then, the following proposition provides a sufficient condition to control $\widetilde{Z}_j^{S_j}$.
\begin{proposition}
	\label{prop: block}
	If $\max(\| W_j^{S_j} \|_\infty, \| R_j^{S_j} \|_\infty)  \leq \frac{\lambda_j \alpha}{4(2- \alpha)}$,  then $| [\widetilde{Z}_j^{S_j}]_{t} | < 1$ for all $t \notin T_j$.
\end{proposition}

Next, we introduce the following three lemmas under Assumptions~\ref{A1Dep},~\ref{A2Inc}, and~\ref{A3Con} to show that conditions in Proposition~\ref{prop: block} hold. For ease of notation, let $\eta = \max\{n,p\}$, $\widetilde{\theta}_{S} = [\widetilde{\theta}_{S_j}]_{T_j}$, $\widetilde{Z}_{S} = [\widetilde{Z}_j^{S_j}]_{T_j}$,  $\widetilde{\theta}_{S^c} = [\widetilde{\theta}_{S_j}]_{S_j \setminus T_j}$, and $\widetilde{Z}_{S^c} = [\widetilde{Z}_j^{S_j}]_{S_j \setminus T_j}$. 

\begin{lemma}
	\label{lem11}
	For any $S_j \in \{ \pi_{1}, \pi_{1:2},..., \pi_{1:j-1} \}$ and $\lambda_j \geq \frac{4 C_{x}^2 \sqrt{2}(2-\alpha) }{  \alpha} \frac{ \log^2 \eta }{ \kappa_1(n,p)  }$ for some $\alpha \in (0, 1]$, 
	\begin{equation*}
	P\left( \frac{\| W_j^{S_j} \|_\infty }{\lambda_j} \leq \frac{\alpha}{4(2- \alpha)} \right) 
	\geq 1 -2 d \cdot \exp \left(- \frac{n}{ \kappa_1(n,p)^2 } \right).
	\end{equation*}
	where $\kappa_1(n,p)$ is an arbitrary function of $n$ and $p$. 
\end{lemma}

\begin{lemma} 
	\label{lem12}
	Suppose that for all $S_j \in \{ \pi_{1}, \pi_{1:2},..., \pi_{1:j-1} \}$, $\|W_j^{S_j}\|_{\infty} \leq \frac{\lambda_j}{4}$. Then, for $\lambda_j \leq \frac{ \rho_{\min}^2 }{ 10 C_{x}^2 \rho_{\max} d  \log^2 \eta }$,
	\begin{equation*}
	 \| \widetilde{\theta}_S - \theta_S^* \|_2 \leq \frac{5}{ \rho_{\min} } \sqrt{d} \lambda_j 
	\end{equation*}
\end{lemma}

\begin{lemma} 
	\label{lem13}
	Suppose that for all $S_j \in \{ \pi_{1}, \pi_{1:2},..., \pi_{1:j-1} \}$, $\|W_j^{S_j}\|_{\infty} \leq \frac{\lambda_j}{4}$. Then, for $\lambda_j \leq \frac{ \alpha \rho_{\min}^2 }{ 100 C_{x}^2 (2 - \alpha) \rho_{\max} d \log^2 \eta  }$ and $\alpha \in (0, 1]$,
	\begin{equation*}
	 \frac{\| R_j^{S_j} \|_\infty }{\lambda_j} \leq \frac{\alpha}{4(2 - \alpha)} 
	\end{equation*}
\end{lemma}

%\textbf{Sion}:
%$$ \frac{64\sqrt{2}(2-\alpha) }{  \alpha} \frac{ \log^2 \eta }{ \kappa_1(n,p)  }$$

The rest of the proof is straightforward using Lemmas~\ref{lem11},~\ref{lem12}, and~\ref{lem13}. Consider the choice of regularization parameter $\lambda_{j0} = \frac{4\sqrt{2}C_{x}^2 (2-\alpha) }{  \alpha} \frac{ \log^2 \eta }{ \kappa_1(n,p)  }$, where 
$\kappa_1(n,p) \geq  \frac{4\sqrt{2} C_{x}^4 \cdot 10^2 (2-\alpha)^2 }{  \alpha^2} \frac{ \rho_{\max} }{ \rho_{\min}^2 } d \log^4 \eta$ 
ensuring that 
$ \frac{4 C_{x}^2 \sqrt{2}(2-\alpha) }{  \alpha} \frac{ \log^2 \eta }{ \kappa_1(n,p)  } \leq  \lambda_{j0}  \leq \frac{ \alpha \rho_{\min}^2 }{ 10^2 C_{x}^2 (2 - \alpha) \rho_{\max} d \log^2 \eta  }$ 
for any $\alpha = (0, 1]$. Hence, if we set $\kappa_1(n,p) = C_{\max} d \log^4 \eta $ where $C_{\max} = \frac{4\sqrt{2} \cdot  10^2 C_{x}^4 \cdot(2-\alpha)^2 }{  \alpha^2} \frac{ \rho_{\max} }{ \rho_{\min}^2 } $, then all conditions for Lemma~\ref{lem11}, \ref{lem12}, and~\ref{lem13} are satisfied. Therefore,
\begin{equation}
\| \widetilde{Z}_{S^c}\|_\infty \leq ( 1- \alpha ) + ( 2- \alpha) \left[ \frac{ \| W_{j}^{S_j} \|_\infty }{ \lambda_j} + \frac{ \| R_{j}^{S_j} \|_\infty }{ \lambda_j} \right] \leq ( 1- \alpha ) + \frac{ \alpha }{4} + \frac{ \alpha }{4} < 1,
\end{equation}
with a probability of at least $1 - 2 d \cdot \exp\left( - \frac{ n }{ \kappa_1(n,p)^2 } \right) = 1 - 2 d \cdot \exp\left( - \frac{ n }{ C_{\max}^2 d^2 \log^8 \eta } \right)$. 

\begin{proposition}
	\label{prop:parents_search}
	Suppose that, for any $j \in V$, partial ordering $(\pi_1,..., \pi_j)$ is correctly estimated. If $ \min_{t \in S} [\theta_S^*]_t \geq \frac{10}{\rho_{\min}} \sqrt{d }~\lambda_j$ for all $j \in V$, 
	$$\mbox{supp}( \widehat{\theta}_{S_{j}} ) = \pa(j).$$ 
\end{proposition}

Proposition~\ref{prop:parents_search} guarantees that $\ell_1$-regularized likelihood regression recovers the parents for each node with a high probability. Since there are $p$ regression problems, for any $\epsilon >0$,  there exists a positive constant $C_{\epsilon} > 0$ such that if $n \geq C_{\epsilon} ( \kappa_1(n,p)^2 \log p )$ for $\kappa_1(n,p) \geq  C_{\max} d \log^4 \eta$, 
\begin{equation*}
P( \widehat{G} = G ) \geq 2 d p \cdot \exp\left( - \frac{ n }{ \kappa_1(n,p)^2 } \right)  \geq 1 - 2 d p \cdot \exp\left(   -  C_{\epsilon} \log p \right) \geq 1- \epsilon.
\end{equation*}

\end{proof}

\section{Proof for Theorem~\ref{ThmMainTheorm}: Ordering Recovery}

\label{SecMainThmProof2}

\begin{proof}
	We begin by reintroducing some necessary notations and definitions to make the proof concise. Without loss of generality, assume that the true ordering is unique and $\pi = (\pi_1, ..., \pi_p) = (1,2,...,p)$. For notational convenience, we define $X_{1:j} = (X_{\pi_1},X_{\pi_2},\cdots,X_{\pi_j}) \\ = (X_1, X_2, ..., X_{j})$ and $X_{1:0} = \emptyset$.  We restate the moments ratio scores for a node $k$ and the $j$th element of the ordering:
	$$
	\S(j, k)  :=\frac{ \E ( X_{k}^2 ) }{ \E( f ( \E( X_{k} \mid X_{1:(j-1)}  )  ) ) } \quad \mbox{and} \quad
	\hatS(j, k)  :=\frac{ \widehat{\E}( X_{k}^2 ) }{ \widehat{\E}( f ( \widehat{\E}( X_{k} \mid X_{\widehat{\pi}_{1:(j-1)}}  )  ) ) },
	$$
	where $f(\mu) := \mu + \mu^2$, $\E( X_k \mid X_{S_k}  ) = \exp( \theta_{k}^{*}  + \sum_{t \in S_k} \theta_{kt}^{*} X_t)$, and $\widehat{\E}( X_k \mid X_{S_k}  ) = \exp( \hat{\theta}_{k}  + \sum_{t \in S_k} \hat{\theta}_{kt} X_t)$ where $\theta_{S_k}^{*} = (\theta_{k}^*, \theta_{kt}^* )$ and $\hat{\theta}_{S_k} = ( \hat{\theta}_{k}, \hat{\theta}_{kt} )$ are the solutions of 
	the problem~\eqref{ThetaSj} and of the $\ell_1$-regularized GLM~\eqref{eqn:theta_jS2}, respectively. In addition, we use the unbiased method-of-moment estimator for a marginal expectation, $\widehat{\E}( X_k^2 ) = \frac{1}{n} \sum_{i = 1}^{n} (X_k^{(i)})^2$ and $\widehat{\E}( f( \widehat{\E}( X_k\mid X_S) ) ) = \frac{1}{n} \sum_{i = 1}^{n} f( \exp( \hat{\theta}_{k}  + \sum_{t \in S_k} \hat{\theta}_{kt} X_t^{(i)}) )$.
	
	We define the following necessary events:
	For each node $j \in V$, $S_j \in \{ \pi_1, \pi_{1:2}, ..., \pi_{1:(j-1)} \}$ and any $\epsilon_1 > 0$; 
	\begin{eqnarray*}
		\zeta_1 & := & \left\{ \max_{j = 1, ...,p-1} \max_{k = j,...,p} \left| \mathcal{S}(j,\pi_k) - \hatS(j,\pi_k) \right| > \frac{M_{\min}}{2} \right\},  \nonumber \\	
		\zeta_2 & := & \left\{ \max_{j \in V} \left| \widehat{\E}( X_j^2)  - \E( X_j^2) \right| < \epsilon_1 \right\}, \nonumber \\
		\zeta_3 & := & \left\{ \max_{j \in V} \left| \widehat{\E}\left( f \left( \widehat{\E}(X_j \mid X_{S_j} )  \right) \right) - {\E}\left( f \left( \widehat{\E}( X_j \mid X_{S_j} ) \right) \right) \right| < \epsilon_1 \right\},  \nonumber \\
		\zeta_4 & := & \left\{ \max_{j \in V} \left| {\E} \left( f \left( \widehat{\E}(X_j \mid X_{S_j} )  \right) \right) - \E \left( f \left( \E( X_j \mid X_{S_j} ) \right) \right)  \right| < \epsilon_1  \right\}.  \nonumber
	\end{eqnarray*}
	%http://www.wolframalpha.com/input/?i=solve+(X+%2B+d1)%2F+(+Y+-+2+d1+)+-+X%2F+(Y+)+-+Z+%3C0,+X%3E0,+d1%3E0,+Y%3E0,+Z%3E0+for+d1
	%http://www.wolframalpha.com/input/?i=solve+X%2F+(Y+)+-++(X+-+d1)%2F+(+Y+%2B+2+d1+)+-+Z+%3C0,+X%3E0,+d1%3E0,+Y%3E0,+Z%3E0+for+d1	
	%http://www.wolframalpha.com/input/?i=solve+X%2F+(Y+)+-++(X+-+d1)%2F+(+Y+%2B+2+d1+)+-+Z+%3C0,+X%3E0,+d1%3E0,+Y%3E0,+Z%3E0,+X%3E+(1%2B2Z)+Y+for+d1
	
	We begin by proving that our algorithm recovers the ordering of a Poisson SEM in the high-dimensional setting. The probability that ordering is correctly estimated from our method can be written as
	\begin{align*}
		& P\left( \widehat{\pi} = \pi \right) \\
		=&P\left( \hatS(1, \pi_1) < \min_{j = 2,...,p} \hatS(1, \pi_j), \hatS(2, \pi_2) < \min_{j =3,...,p} \hatS(2, \pi_j), ..., \hatS(p-1, \pi_{p-1}) < \hatS(p-1, \pi_p) \right) \\
		 =& P\left( \min_{j = 1, ...,p-1} \min_{k = j+1,...,p} \hatS(j, \pi_k) -\hatS(j, \pi_j) > 0 \right)\\
%		 =& P\bigg( \min_{j = 1, ...,p-1} \min_{k = j+1,...,p} \left\{ \Big( \mathcal{S}(j, \pi_k) -\mathcal{S}(j, \pi_j)\Big)  
%		\right.
%		\\
%		&\hspace{2.03in} \left.
%		- \left( \mathcal{S}(j,\pi_k) - \hatS(j,\pi_k) \right) + \left( \mathcal{S}(j,\pi_j) - \hatS(j,\pi_j) \right) \right\} > 0  \bigg) \\
		=& P\bigg( \min_{\substack{j = 1, ...,p-1\\k = j+1,...,p}} \left\{ \Big( \mathcal{S}(j, \pi_k) -\mathcal{S}(j, \pi_j)\Big)  
		- \left( \mathcal{S}(j,\pi_k) - \hatS(j,\pi_k) \right) + \left( \mathcal{S}(j,\pi_j) - \hatS(j,\pi_j) \right) \right\} > 0  \bigg) 
		\\
		 \geq & P\left( \min_{ \substack{ j = 1, ...,p-1 \\ k = j+1,...,p} } \left\{ \left( \mathcal{S}(j, \pi_k) -\mathcal{S}(j, \pi_j) \right) \right\} > M_{\min} \right.,and~ \left. \max_{ \substack{ j = 1, ...,p-1 \\ k = j,...,p} } \left| \mathcal{S}(j,\pi_k) - \hatS(j,\pi_k) \right| < \frac{M_{\min}}{2}  \right).
	\end{align*}
	
	The first term in the above probability is always satisfied because $ \mathcal{S}(j,\pi_k) - \mathcal{S}(j,\pi_j) > ( 1+ M_{\min} ) - 1 = M_{\min}$ from Assumption~\ref{A1}.  Hence, the lower bound of the probability that ordering is correctly estimated using our method is reduced to
	\begin{eqnarray}
	\label{eqn:mainlowerbound}
	P\left( \widehat{\pi} = \pi \right) 
	&\geq& P\left( \max_{j = 1, ...,p-1} \max_{k = j,...,p} \left| \mathcal{S}(j,\pi_k) - \hatS(j,\pi_k) \right| < \frac{M_{\min}}{2}  \right) \nonumber \\
	&= & 1 - P(\zeta_1) \nonumber \\
	&= & 1 - P(\zeta_1 \mid \zeta_2, \zeta_3, \zeta_4) P(\zeta_2, \zeta_3, \zeta_4) - P(\zeta_1 \mid ( \zeta_2,\zeta_3,\zeta_4 )^c ) P( (\zeta_2, \zeta_3,\zeta_4)^c ) \nonumber  \\
	&\geq & 1 - P(\zeta_1 \mid \zeta_2, \zeta_3, \zeta_4) - P( (\zeta_2, \zeta_3, \zeta_4)^c ) \nonumber \\
	&\geq & 1 - \underbrace{ P(\zeta_1 \mid \zeta_2, \zeta_3, \zeta_4) }_{Lem~\ref{lemma1}} - \underbrace{ P(\zeta_2^c ) - P(\zeta_3^c ) - P(\zeta_4^c ) }_{Lem~\ref{lemma123}}.
	\end{eqnarray}
	
	Next, we introduce the following two lemmas to show the lower bound of the probability in~\eqref{eqn:mainlowerbound} as a function of the triple $(n, p,d)$:
	%	The first lemma shows that the estimated score is accurate under some regularity conditions. 
	\begin{lemma}
		\label{lemma1}
		Given the sets $\zeta_2, \zeta_3, \zeta_4$ and under Assumption~\ref{A1}, $P(\zeta_1 \mid \zeta_2, \zeta_3, \zeta_4) = 0$ if for some small $\epsilon_1$ such that for any $S_j \in \{ \pi_1, \pi_{1:2}, ..., \pi_{1:(j-1)} \}$, 
		$$
		\epsilon_1 <  \min\left\{ \frac{ \E(X_j^2) M_{\min} }{ 2 ( M_{\min} + 3 )( M_{\min} + 1 ) },  \frac{M_{\min}}{6} \frac{  \E( f( \E( X_j \mid X_{S_j} ) ) )^2 }{ \E( X_j^2 )  }  \right\}, 
		$$
		where $f(\mu) = \mu + \mu^2$. 
	\end{lemma}
	The condition in Lemma~\ref{lemma1} implies that if $\epsilon_1$ is sufficiently small, the estimated score is close to the true score value.
	
	The second lemma shows the error bound for the consistency of the estimators. 
	\begin{lemma}
		\label{lemma123}
		For any $\epsilon_1 > 0$ and
		\begin{itemize}
			\item[(i)] For $\zeta_2$, $P(\zeta_2^c) \leq 1- 2 \cdot p \cdot \exp \left\{ - \frac{  n \epsilon_1^2 }{  2 C_{x}^4  \log^{4} \eta } \right\}.$
			\item[(ii)] For $\zeta_3$, there exist some positive constants $C_{\max}$ and $D_{\max}$ such that 
			\begin{eqnarray*}
				P(\zeta_3^c ) \leq 1 -2 \cdot p \cdot d \cdot &\exp \left(- \frac{n}{ \kappa_1(n,p)^2 } \right)- 2 \cdot p \cdot \exp \left\{ - \frac{  n \epsilon_1^2 }{ D_{\max} \log^4 \eta } \right\}.
			\end{eqnarray*}
			where $\kappa_1(n,p) \geq  C_{\max} d \log^4 \eta$.
			\item[(iii)] For $\zeta_4$, $P(\zeta_4^c)  = 0.$
		\end{itemize}
	\end{lemma}
	
	Therefore, we complete the proof: our method recovers the true ordering at least of 
	\begin{eqnarray*}
		P\left( \widehat{\pi} = \pi \right) &\geq & 1 -C_1 p \cdot  d \cdot \exp \left(- C_2 \frac{n}{ \kappa_1(n,p)^2 } \right).
	\end{eqnarray*}
	for $\kappa_1(n,p) \geq  C_{\max} d \log^4 \eta$, and some positive constants $C_1$ and $C_2$
	
\end{proof}

\section{Proposition~\ref{prop:solution} }

We begin by introducing an important proposition to control the tail behavior for the distribution of each node, which are required to prove the lemmas. 

%%%%%%%%%%%%%%%%%%%%%%%%%%%

\begin{proposition}
	\label{prop:solution}
	For given $j \in V$ and $S_j \in \{ \pi_{1} , \pi_{1:2},..., \pi_{1:(j-1)} \}$,  the solution $\widehat{\theta}_{S_j}$ in Equation~\eqref{eq:Objective} satisfies
	\begin{eqnarray*}
		\frac{1}{n} \sum_{i = 1}^{n} \exp\left( \langle \widehat{\theta}_{S_j}, X_{S_j}^{(i)} \rangle \right) < C_{x} \log \eta .
	\end{eqnarray*}
	where $C_{x} >2$ is a constant in Assumption~\ref{A3Con}. 
\end{proposition}

\begin{proof}
	By the first-order optimality condition of $\mathcal{L}_j^{S_j}(\theta_{S_j}, X^{1:n} )$ in Equation \eqref{eq:Objective}, we have
	\begin{eqnarray*}
		\sum_{i = 1}^{n} X_j^{(i)} & = & \sum_{i = 1}^{n} \exp( \langle \hat{\theta}_{S_j}, X_{S_j}^{(i)} \rangle )  \\
		\sum_{i = 1}^{n} X_j^{(i)}X_k^{(i)} & = & \sum_{i = 1}^{n} \exp\left( \langle \hat{\theta}_{S_j}, X_{S_j}^{(i)} \rangle \right)X_k^{(i)}  + \lambda_j \text{sign}( [\hat{\theta}_{S_j}]_{k} ).  
	\end{eqnarray*}
	
	By Assumption~\ref{A3Con}, we have
	\begin{eqnarray*}
		\frac{1}{n} \sum_{i = 1}^{n} \exp\left( \langle \hat{\theta}_{S_j}, X_{S_j}^{(i)} \rangle \right) \leq C_{x} \log \eta \iff
		\frac{1}{n} \sum_{i = 1}^{n} X_j^{(i)}  \leq C_{x} \log \eta.
	\end{eqnarray*}
\end{proof}

\section{Proof for Propositions~\ref{prop: block} and~\ref{prop:parents_search} }

\subsection{Proof for Proposition~\ref{prop: block}}

%We omit the superscript and subscript $S_j$ to improve readability.

\begin{proof}
	We note that $\widetilde{\theta}_{S^c} = (0,0,...,0)^T \in \mathbb{R}^{|S^c|}$ in our primal-dual construction. To improve readability, we let $\theta_{S} = [\theta_{S_j}]_{T_j}, \theta_{S^c} = [\theta_{S_j}]_{S_j \setminus T_j}$ , and $A_S  = [A_j^{S_j}]_{T_j}$ and $A_{S^c}  = [A_j^{S_j}]_{S_j \setminus T_j}$. With these notations, $W_{S}$ and $R_{S}$ are sub-vectors of $W_{j}^{S_j}$ and $R_{j}^{S_j}$ corresponding to variables $X_S$, respectively.
	
	We can restate condition~\eqref{eq:Contraint2} in block form as follows:
	\begin{eqnarray*}
		Q_{S^c S}[ \widetilde{\theta}_{S} - \theta_{S}^*] & = & W_{S^c} - \lambda_j \widetilde{Z}_{S^c} + R_{S^c}, \\
		Q_{S S}[ \widetilde{\theta}_S - \theta_S^*] & = & W_{S} - \lambda_j \widetilde{Z}_{S} + R_{S}.
	\end{eqnarray*}
	
	Since $Q_{S S}$ is invertible, the above equations can be rewritten as
	\begin{equation*}
	Q_{S^c S} Q_{S S}^{-1} [ W_{S} - \lambda_j \widetilde{Z}_{S} - R_{S} ] = W_{S^c} - \lambda_j \widetilde{Z}_{S^c} - R_{S^c}.
	\end{equation*}
	
	Therefore,
	\begin{equation*}
	[W_{S^c} - R_{S^c} ] - Q_{S^c S} Q_{S S}^{-1} [ W_{S}  - R_{S}] + \lambda_j Q_{S^c S} Q_{S S}^{-1} \widetilde{Z}_{S} =  \lambda_j \widetilde{Z}_{S^c}.
	\end{equation*}
	
	Taking the $\ell_\infty$ norm of both sides yields
	\begin{eqnarray*}
		\| \widetilde{Z}_{S^c}\|_\infty & \leq & |\| Q_{S^c S} Q_{S S}^{-1} \||_{\infty} \left[ \frac{ \| W_{S} \|_\infty }{ \lambda_j} + \frac{ \| R_{S} \|_\infty }{ \lambda_j} +1 \right] + \frac{ \| W_{S^c} \|_\infty }{ \lambda_j} + \frac{ \| R_{S^c} \|_\infty }{ \lambda_j}.
	\end{eqnarray*}
	
	Recalling Assumption~\eqref{A2Inc}, we obtain $ |\| Q_{S^c S} Q_{S S}^{-1} \||_{\infty} \leq (1 - \alpha)$, and hence, we have 
	\begin{eqnarray*} 
		\| \widetilde{Z}_{S^c}\|_\infty & \leq & (1- \alpha) \left[ \frac{ \| W_{S} \|_\infty }{ \lambda_j} + \frac{ \| R_{S} \|_\infty }{ \lambda_j} +1 \right] + \frac{ \| W_{S^c} \|_\infty }{ \lambda_j} + \frac{ \| R_{S^c} \|_\infty }{ \lambda_j} \\
		&\leq & ( 1- \alpha ) + ( 2- \alpha) \left[ \frac{ \| W_{j}^{S_j} \|_\infty }{ \lambda_j} + \frac{ \| R_{j}^{S_j} \|_\infty }{ \lambda_j} \right].
	\end{eqnarray*}
	
	If both $\| W_{j}^{S_j} \|_\infty$ and $\| R_{j}^{S_j} \|_\infty$ are less than $\frac{\lambda_j \alpha}{4(2- \alpha)}$, as assumed, then
	$$
	\| \widetilde{Z}_{S^c}\|_\infty \leq ( 1- \alpha ) + \frac{\alpha}{2} < 1.
	$$
\end{proof}

\subsection{Proof for Proposition~\ref{prop:parents_search}}

\begin{proof}
	To prove the support of $\hat{\theta}_S$ is not strictly subset the true
	support $X_S$, it is sufficient to show that the maximum bias is bounded:
	$$
	\|\widehat{\theta}_{S} - \theta_{S}^* \|_{\infty} \leq \frac{ \min_{t \in S} [\theta_S^*]_{t} }{2}.
	$$
	From Lemma~\ref{lem12}, we have, with a high probability, 
	$$
	\|\widehat{\theta}_{S} - \theta_{S}^* \|_{\infty} \leq \|\widehat{\theta}_{S} - \theta_{S}^* \|_{2} \leq \frac{5}{\rho_{\min}} \sqrt{d }~\lambda_j. 
	$$
	Therefore, if $\min_{t \in S} [\theta_S^*]_{t} \geq \frac{10}{\rho_{\min}} \sqrt{d }~\lambda_j$,
	$$
	\|\widehat{\theta}_{S} - \theta_{S}^* \|_{\infty} \leq \frac{ \min_{t \in S} [\theta_S^*]_{t} }{2}.
	$$
\end{proof}

\section{Proof for Lemmas}

\subsection{Proof for Lemma~\ref{lemma:uniqueness}}

\begin{proof}
	This lemma can be proved by the same manner developed for the special cases~\citep{wainwright2006high, ravikumar2011high}. In addition, this proof is directly from Lemma 8 in \cite{yang2015graphical}. And, we restate the proof in our framework. The main idea of the proof is the \emph{primal-dual-witness} method which asserts that there is a solution to the dual problem $\widetilde{\theta}_{S_j} = \widehat{\theta}_{S_j}$ if the following Karush-Kuhn-Tucker (KKT) conditions are satisfied. 
	\begin{itemize}
		\item[(a)] We define $\widetilde{\theta}_{S_j} \in \Theta_{S_j}$, where $\Theta_{S_j} = \{ \theta \in \mathbb{R}^{|S_j|} : \theta_{S^c} = 0 \}$ is the solution to the following optimization problem: 
		\begin{equation}
		\label{eq:ResObjective}
		\widetilde{\theta}_{S_j} := \arg \min_{\theta \in \Theta_{S_j}} \mathcal{L}_j^{S_j}( \theta, \lambda_j ) = \arg \min_{\theta \in \Theta_{S_j} } \{ \ell_j^{S_j}(\theta ;X^{1:n}) + \lambda_j \| \theta \|_1 \}.
		\end{equation}
		\item[(b)] Define $\widetilde{Z}_j^{S_j}$ to be a sub-differential for the regularizer $\| \cdot \|_1$ evaluated at $\widetilde{\theta}_{S_j}$. For any $t \in T_j$ in Equation~\eqref{eq:subparents}, $[\widetilde{Z}_j^{S_j}]_{t} = \mbox{sign}([\widetilde{\theta}_{S_j}]_{t})$. 
		\item[(c)] For any $t \notin T_j$, $| [\widetilde{Z}_j^{S_j}]_{t} | < 1$.
	\end{itemize}
	
	If conditions (a) to (c) are satisfied, $\widetilde{\theta}_{S_j} = \widehat{\theta}_{S_j}$ meaning that the solution to unrestricted problem \eqref{eq:Objective} is the same as the solution to restricted problem \eqref{eq:ResObjective} (See \citealp{ravikumar2011high} for details).
	
	%We note that Conditions (a) to (b) suffice to obtain a pair $(\widetilde{\theta}_{S_j}, \widetilde{Z}_{S_j})$ that satisfies the optimality condition as in \eqref{eq:Contraint2}; however Condition (c) is does not guarantee that $\widetilde{Z}_j^{S_j}$ is an element of the sub-differential $\| \widetilde{\theta}_{S_j} \|_1$.
	
	In addition, if the sub-matrix of the Hessian $Q_{SS}^{S_j}$ is invertible, restricted problem \eqref{eq:ResObjective} is strictly convex, and hence, $\widetilde{\theta}_{S_j}$ is unique.
\end{proof}

\subsection{Proof for Lemma~\ref{lem11}}

\label{SubSecProofLem11}

\begin{proof}
	In order to improve readability, we omit the superscript $S_j$ if it is understood (i.e., $W_{j} = W_j^{S_j}$). Each entry of the sample score function $W_{j}$ in Equation~\eqref{eq:Wn} has the form $W_{jt} = \frac{1}{n} \sum_{i=1}^n W_{jt}^{(i)}$ for any $t \in S := \{k \in S_j \mid [\theta_{S_j}^{*}]_{k} \neq 0 \}$. In addition, $W_{jt} = 0$ for all $t \notin S$, since $[\theta_{S_j}^*]_t = 0$ by the definition of $S$.
	 
	Hence simple calculation yields that, for any  $t \in S$ and $i \in \{ 1,2,\cdots,n\}$, 
	$$W_{jt}^{(i)} = X_{t}^{(i)} X_j^{(i)} - \exp( \langle \theta_{S}^*, X_{S}^{(i)} \rangle) X_t^{(i)},$$
	and $(|W_{jt}^{(i)}|)_{i = 1}^{n}$ has mean $0$ by the first-order optimality condition, 
	$
	\E(X_j)  =  \E( \exp( \langle \theta_{S}^*, X_{S} \rangle ) ).
	$
	
	Now, we show that $W_{jt}^{(i)}$ is bounded with a high probability given Assumption~\ref{A3Con} by using Hoeffding's inequality. The both terms are bounded above $C_{x}^2 \log^2 \eta$ by Assumption~\ref{A3Con}. Therefore, $|W_{jt}^{(i)}|$ is bounded by $2 C_{x}^2 \log^2 \eta$.
	
	Applying the union bound and Hoeffding's inequality, we have
	\begin{equation*}
	P( \| W_{j} \|_\infty > \delta ) 
	\leq d \cdot \max_{t \in S} P( | W_{jt} | > \delta)
	\leq 2 d \cdot \exp \left( - \frac{ 2 n \delta^2}{  4 C_{x}^4 \log^4 \eta  } \right).
	\end{equation*}
	
	Suppose that $\delta = \frac{\lambda_j \alpha}{4(2- \alpha)}$ and $\lambda_j \geq \frac{4(2-\alpha) }{\alpha} \frac{ 2 C_{x}^2 \log^2 \eta }{ \sqrt{2} \kappa_1(n,p)  }$. Then, we complete the proof:
	\begin{align}
	\label{eq:probWn}
	P\left( \frac{\| W_{j} \|_\infty }{\lambda_j} > \frac{\alpha}{4(2- \alpha)} \right) 
	\leq 2 d \cdot \exp \Big(- \frac{\alpha^2}{16 (2- \alpha)^2} \frac{ 2 n \lambda_j^2}{ 4 C_{x}^4 \log^4 \eta  } \Big)
	\leq  2 d \cdot \exp \left(-\frac{n}{ \kappa_1(n,p)^2 } \right).
	\end{align}
	
\end{proof}

\subsection{Proof for Lemma~\ref{lem12}}

\label{SubSecProofLem12}

\begin{proof}
	In order to establish error bound $\| \widetilde{\theta}_{S} - \theta_{S}^*\| \leq B$ for some radius $B$, several works \citep{ meinshausen2006high, wainwright2006high, ravikumar2011high,yang2015graphical, park2017learning} already proved that it suffices to show $F(u_{S}) > 0 $ for all $u_{S}:= \widetilde{\theta}_{S} - \theta_{S}^*$ such that $\| u_{S} \|_2 = B$ where 
	\begin{equation}
	\label{eq:F}
	F(a) := \ell_j( \theta_{S}^* + a; X^{1:n}) - \ell_j( \theta_{S}^*; X^{1:n}) + \lambda_j( \| \theta_{S}^* + a \|_1 - \| \theta_{S}^* \|_1 ).
	\end{equation}
	
	More specifically, since $u_{S}$ is the minimizer of $F$ and $F(0) = 0$ by the construction of Equation~\eqref{eq:F}, $F(u_{S}) \leq 0$. Note that $F$ is convex, and therefore we have $F(u_{S}) < 0$. 
	
	Next we claim that $\| u_{S} \|_2 \leq B$. In fact, if $u_{S}$ lies outside the ball of radius $B$, then there exists $v \in (0,1)$ such that the convex combination $v \cdot u_{S} + (1-v) \cdot 0$ would lie on the boundary of the ball. However it contradicts the assumed strict positivity of $F$ on the boundary because, by convexity, 
	\begin{equation}
	F(v \cdot u_{S} + (1-v) \cdot 0 ) \leq v \cdot F(u_{S}) + (1-v) \cdot 0 \leq 0.
	\end{equation}
	
	Thus it suffices to establish strict positivity of $F$ on the boundary of the ball with radius $B := M_{B} \lambda_j \sqrt{ d }$ where $M_{B} > 0$ is a parameter to be chosen later in the proof. Let $u_{S} \in \mathbb{R}^{|S|}$ be an arbitrary vector with $\|u_{S}\|_2 = B$. By the Taylor series expansion of $F$ in~\eqref{eq:F},
	\begin{equation}
	\label{eq:TaylorF}
	F(u_{S}) = (W_{S})^T u_S + u_{S}^T [\bigtriangledown^2 \ell_j ( \theta_{ S }^* + v u_{S} ; X^{1:n} ) ] u_S + \lambda_j( \| \theta_{S}^* + u_{S} \|_1 - \| \theta_{S}^* \|_1 ), 
	\end{equation}
	for some $v \in [0,1]$. 
	
	The first term in Equation~\eqref{eq:TaylorF} has the following bound: applying $\| W_{S} \|_{\infty} \leq \frac{\lambda_j}{4}$ by assumption and $\| u_S \|_1 \leq \sqrt{d} \| u_S \|_2 \leq \sqrt{d} \cdot B$,
	\begin{equation*}
	|(W_{S})^T u_{S}| \leq \| W_{S}\|_{\infty} \| u_{S} \|_1 \leq \|W_{S} \|_{\infty} \sqrt{ d } \|u_{S}\|_2 \leq (\lambda_j \sqrt{ d } )^2 \frac{M_{B}}{4}.
	\end{equation*}
	
	The third term in Equation~\eqref{eq:TaylorF} has the following bound: Applying the triangle inequality,
	\begin{equation*}
	\lambda_j ( \| \theta_{S}^* + u_{S} \|_1 - \| \theta_{S}^*\|_1 ) \geq - \lambda_j \|u_{S}\|_1 \geq -\lambda_j \sqrt{ d } \|u_{S}\|_2 = -M_{B} ( \lambda_j \sqrt{ d } )^2.
	\end{equation*}
	
	Now we show the bound for the second term using  the minimum eigenvalue of a matrix $\bigtriangledown^2  \ell_j( \theta_{S}^* + v u_{S})$:
	\begin{align}
	\label{eq:q}
	q^* &: = \lambda_{\min}\left(\bigtriangledown^2  \ell_j( \theta_{S}^* + v u_{S}) \right) \nonumber \\
	& \geq \min_{ v \in [0,1] } \lambda_{\min}\left( \bigtriangledown^2 \ell_j ( \theta_{S}^* + v u_{S}) \right) \nonumber  \\
	& \geq \lambda_{\min} \left( \bigtriangledown^2 \ell_j ( \theta_{S}^* ) \right) 
	- \max_{v \in [0, 1]} \left\| \frac{1}{n} \sum_{i=1}^n \exp\big( \langle \theta_{S}^* + v u_{S}, X_{S}^{(i)} \rangle \big) u_{S}^T X_{S}^{(i)}  X_{S}^{(i)} (X_{S}^{(i)})^T \right\|_2  \nonumber  \\	
	&\geq \rho_{\min} - \max_{v \in [0, 1]} \max_{y: \|y\|_2 = 1} \frac{1}{n} \sum_{i=1}^n \exp\big( \big\langle \theta_{S}^* + v u_{S}, X_{S}^{(i)} \big\rangle \big) \cdot \big( y^T X_{S}^{(i)} \big)^2 \cdot \big| u_{S}^T X_{S}^{(i)} \big|.
	\end{align}
	
	We first show the bound of the first term in Equation~\eqref{eq:q}: Note that $\theta_{S}^* + v u_{S}$ is a linear (convex) combination of $\theta_{S}^*$ and $\widetilde{\theta}_{S}$. Hence, by Assumption~\ref{A3Con} and Proposition~\ref{prop:solution}, we obtain 
	$$ 
	 \frac{1}{ n} \sum_{i = 1}^{n} \exp\big( \langle \theta_{S}^* + v u_{S}, X_{S}^{(i)} \rangle \big) \leq C_{x} \log \eta.
	$$
	
	Now, we bound the second term in Equation~\eqref{eq:q}: Recall that $\|X_{S}^{(i)}\|_{\infty} \leq C_{x} \log \eta$ for all $i$ by Assumption~\ref{A3Con}. Recall $[u_{S}]_t = 0$ for $t \notin S$ by the primal-dual construction of~\eqref{eq:Contraint2}. Applying $\|u_{S}\|_1 \leq \sqrt{ d } \|u_{S}\|_2 \leq \sqrt{d}\cdot B$,
	\begin{equation*}
	\big| u_{S}^T X_{S}^{(i)} \big| \leq C_{x} \log(\eta) \sqrt{ d } \|u_{S}\|_2 \leq C_{x} \log(\eta) \cdot M_{B} \lambda_j  d.
	\end{equation*}
	
	Lastly, it is clear that  $\max_{y: \|y\|_2 = 1} ( y^T X_{S}^{(i)} )^2 \leq  \rho_{\max}$ by the definition of the maximum eigenvalue and Assumption~\ref{A1Dep}. Together with the above bounds, we obtain
	\begin{equation*}
	P\left( q^* \leq \rho_{\min} - C_{x}^2 M_{B} \rho_{\max} d  \lambda_j \log^2 \eta \right) \leq  M \eta^{-2}.
	\end{equation*}
	
	For $\lambda_j \leq \frac{ \rho_{\min} }{ 2 C_{x}^2 M_{B} \rho_{\max} d  \log^2 \eta }$, we have $q^* \geq \frac{ \rho_{\min} }{2} $ with a high probability. Therefore, 
	\begin{equation*}
	F(u) \geq (\lambda_j \sqrt{n})^2 \Big\{ -\frac{1}{4} M_{B} + \frac{ \rho_{\min}}{2} M_{B}^2 - M_{B} \Big\},
	\end{equation*}
	which is strictly positive for $M_{B} = \frac{5}{ \rho_{\min} }$. Therefore, for $\lambda_j \leq \frac{ \rho_{\min}^2 }{ 10 C_{x}^2 \rho_{\max} d  \log^2 \eta }$,
	\begin{equation*}
	 \| \widetilde{\theta}_S - \theta_S^* \|_2 \leq \frac{5}{ \rho_{\min} } \sqrt{ d } \lambda_j.
	\end{equation*}
	
\end{proof}

\subsection{Proof for Lemma~\ref{lem13}}

\begin{proof}
	To improve readability, we use $R_S = [R_{j}^{S_j}]_{S}$ where $S := \{k \in S_j \mid [\theta_{S_j}^{*}]_{k} \neq 0 \}$. Then, each entry of $R_{j}^{S_j}$ in Equation~\eqref{eq:Rn} has the form $R_{jk} = \frac{1}{n} \sum_{i=1}^n R_{jk}^{(i)}$ for any $k \in S_j$, and it can be expressed as
	\begin{eqnarray*}
		R_{jk} 
		& = & \frac{1}{n} \sum_{i=1}^{n} [ \bigtriangledown^2 \ell_j(\theta_{S_j}^*; X^{1:n}) - \bigtriangledown^2 \ell_j(\bar{\theta}_{S_j}; X^{1:n})]_k^T (\widetilde{\theta}_{S_j} - \theta_{S_j}^*) \\
		& = & \frac{1}{n} \sum_{i=1}^{n} \left[\exp\left( \left\langle \theta_{S}^*, X_{S}^{(i)} \right\rangle \right) - \exp\left( \left\langle \bar{\theta}_{S} , X_{S}^{(i)} \right\rangle \right) \right] \left[ X_{S}^{(i)}(X_{S}^{(i)})^T \right]_k^T \left(\widetilde{\theta}_{S} - \theta_{S}^*\right)
	\end{eqnarray*}
	for $\bar{\theta}_{S}$, which is a point on the line between $\widetilde{\theta}_{S}$ and $\theta_{S}^*$ (i.e., $\bar{\theta}_{S}^{(t)} = v \cdot \widetilde{\theta}_{S} + (1 - v) \cdot \theta_{S}^*$ for some $v \in [0,1]$). The second equality holds because $\theta_{S^c}^* = \tilde{\theta}_{S^c} = (0,0,...,0) \in \mathbb{R}^{|S^c|}$. 
	
	Applying the mean value theorem again, we have,
	\begin{equation*}
	R_{jk}  
	= \frac{1}{n} \sum_{i=1}^n \Big\{ \exp\left( \left\langle \bar{ \bar{ \theta} }_{S} , X_{S}^{(i)} \right\rangle \right)  X_{k}^{(i)} \Big\} \Big\{ v ( \widetilde{\theta}_{S} - \theta_{S}^* )^T X_{S}^{(i)}  (X_{S}^{(i)})^T ( \widetilde{\theta}_{S} - \theta_{S}^* ) \Big\}
	\end{equation*}
	for $\bar{ \bar{ \theta} }_{S_j} $ which is a point on the line between $\bar{\theta}_{S_j} $ and $\theta_{S_j}^*$ (i.e., $\bar{ \bar{\theta} }_{S_j}  = v \cdot \bar{\theta}_{S_j} + (1 - v) \cdot \theta_{S_j}^*$ for $v \in [0,1]$). 
	
	Note that $\bar{ \bar{ \theta} }_{S_j} $ is a linear (convex) combination of $\theta_{S}^*$ and $\widetilde{\theta}_{S}$. Hence, from Assumption~\ref{A3Con} and Proposition~\ref{prop:solution}, we obtain,
	$$ 
	 \frac{1}{ n} \sum_{i = 1}^{n} \exp\left( \left\langle \bar{ \bar{ \theta} }_{S_j} , X_{S_j}^{(i)} \right\rangle \right) \leq C_{x} \log \eta, \quad \text{and} \quad \max_{i, j} X_{j}^{(i)} < C_{x} \log \eta.
	$$
	Therefore, we have $| R_{jk} |	\leq  \rho_{\max} C_{x}^2 \log^2 \eta \|\widetilde{\theta}_{ S } - \theta_{ S }^* \|_2^2$ for all $j, k \in V$. 
	
	In Section~\ref{SubSecProofLem12}, we showed that $\|\widetilde{\theta}_{ S } - \theta_{ S }^* \|_2 \leq \frac{5}{ \rho_{\min} } \sqrt{ d } \lambda_j$ for $\lambda_j \leq \frac{ \rho_{\min}^2 }{ 10 C_{x}^2 \rho_{\max} d  \log^2 \eta }$. Therefore, if $\lambda_j \leq \frac{ \rho_{\min}^2 }{ 25 \cdot C_{x}^2  \rho_{\max} d \log^2 \eta   } \frac{\alpha}{4(2 - \alpha)} $, we obtain,
	\begin{eqnarray*}
		P \left( \| R_{j} \|_\infty >  \frac{\alpha}{4(2 - \alpha)} \lambda_j  \right) 
		\leq P \left( \| R_{j} \|_\infty > 25 C_{x}^2 \lambda_j^2  \frac{ \rho_{\max} }{ \rho_{\min}^2 } d \log^2 \eta   \right)	= 0.
	\end{eqnarray*}
	Therefore, we have, 
	$$
	 \| R_{j} \|_\infty \leq  \frac{\alpha}{4(2 - \alpha)} \lambda_j 
	 $$
\end{proof}

\subsection{Proof for Lemma \ref{lemma1} }
\begin{proof}
	Conditioning on the sets~$\zeta_2, \zeta_3$, and $\zeta_4$, we provide the following results for different two cases: 
	
	(i) For any $j \in \{1,2,...,p-1\}$, and $X_S = X_{1:(j-1)}$, we have $\frac{ \E( X_j^2 ) }{ \E( f( \E( X_j \mid X_S ) ) )} = 1$. Therefore, for $k = \pi_j$, we have the following probability bound:
	\begin{eqnarray*}
		& & P\left( |\hatS(j,k) - \mathcal{S}(j,k)| < \frac{ M_{min} }{ 2 } \Big\vert \zeta_2, \zeta_3,  \zeta_4 \right) \\
		& = & P\left( \left| \frac{ \widehat{\E}( X_k^2 ) }{ \widehat{\E}( f( \widehat{\E}( X_k \mid X_S ) ) )} - \frac{ \E( X_k^2 ) }{ \E( f( \E( X_k \mid X_S ) ) )} \right| < \frac{ M_{min} }{ 2 }  \Big\vert \zeta_2, \zeta_3, \zeta_4 \right) \\
		&\geq & P \bigg( \frac{ \E( X_k^2 ) + \epsilon_1 }{ \E( f( \E( X_k \mid X_S ) ) ) - 2 \epsilon_1} - \frac{ \E( X_k^2 ) ) }{ \E( f( \E( X_k \mid X_S ) ) ) } < \frac{ M_{min} }{ 2 } ~and~\\
		&& \quad \quad \quad \quad \quad \frac{ \E( X_k^2 ) ) }{ \E( f( \E( X_k \mid X_S ) ) ) } - \frac{ \E( X_k^2 ) ) -\epsilon_1 }{ \E( f( \E( X_k \mid X_S ) ) ) +2\epsilon_1}  < \frac{ M_{min} }{ 2 }  \bigg) \\
		&\geq & P \bigg( \epsilon_1 < \frac{  \E( f( \E( X_k \mid X_S ) ) ) M_{\min} }{ 2 ( M_{\min} + 3 ) } \bigg)\\
		&\geq & P \bigg( \epsilon_1 < \frac{ \E(X_k^2) M_{\min} }{ 2 ( M_{\min} + 3 )( M_{\min} + 1 ) }	\bigg).
	\end{eqnarray*}
	
	(ii) For $j \in \{1,2,...,p-1\}$, $k \in \{ \pi_{j+1},..., \pi_{p} \}$ having parent $\pi_j$, and $X_S = X_{1:(j-1)}$, we have $ \E( X_k^2 ) > (1 + M_{\min}) \E( f( \E( X_k \mid X_S ) ) )$ by Assumption~\ref{A1}. In addition, some elementary but complicated computations yield 	 
	\begin{eqnarray*}
		& & P\left( \big|\hatS(j,k) - \mathcal{S}(j,k) \big| < \frac{ M_{min} }{ 2 } \Big\vert \zeta_2, \zeta_3,  \zeta_4 \right) \\
		&\geq& P \left( \epsilon_1 < \frac{  \E( f( \E( X_k \mid X_S ) ) )^2 M_{\min} }{ 4 \E( X_k^2 ) + 2 \E( f( \E( X_k \mid X_S ) ) ) + 2 \E( f( \E( X_k \mid X_S ) ) ) M_{\min} } \right) \\
		&\geq& P \left( \epsilon_1 < \frac{  \E( f( \E( X_k \mid X_S ) ) )^2 M_{\min}(1+ M_{\min}) }{ 4 \E( X_k^2 ) (1+ M_{\min}) + 2 \E( X_k^2 )+ 2 M_{\min} \E( X_k^2 )  } \right) \\
		&\geq& P \left( \epsilon_1 < \frac{  \E( f( \E( X_k \mid X_S ) ) )^2 M_{\min}(1+ M_{\min}) }{ 6 (1+ M_{\min})  \E( X_k^2 )  } \right) \\
		&\geq& P \left( \epsilon_1 < \frac{M_{\min}}{6} \frac{  \E( f( \E( X_k \mid X_S ) ) )^2 }{ \E( X_k^2 )  } \right).
	\end{eqnarray*}
	
	Therefore $P(\zeta_1 \mid \zeta_2, \zeta_3, \zeta_4) = 0$ if $\epsilon_1$ is sufficiently small enough. For any node $j$, any set $S_j \in \{ \pi_1, \pi_{1:2}, ..., \pi_{1:(j-1)} \}$, and $k \in \{\pi_{j}, \pi_{j+1},..., \pi_{p} \}$,
	{\small
		$$
		\epsilon_1 <  \min\left\{ \frac{ \E(X_k^2) M_{\min} }{ 2 ( M_{\min} + 3 )( M_{\min} + 1 ) },  \frac{M_{\min}}{6} \frac{  \E( f( \E( X_k \mid X_{S_j} ) ) )^2 }{ \E( X_k^2 )  }  \right\}. 
		$$
	}
\end{proof}

\subsection{Proof for Lemma \ref{lemma123}}

The proof for Lemma~\ref{lemma123} is closely related to the proof in Appendix \ref{SecThmMainProof1}. Hence, for brevity, we do not present the details of the proof already shown in Appendix~\ref{SecThmMainProof1}. 

\begin{itemize}
	\item[(i)] $P(\zeta_2^c) \leq   2 p \cdot \exp \left\{ - \frac{  n \epsilon_1^2 }{  2 C_{x}^4  \log^{4} \eta } \right\}.$
	\begin{proof}
		Using Hoeffding's inequality given Assumption~\ref{A3Con}, for any $\epsilon>0$ and $j \in V$, 
		\begin{equation}
		P\left( \left| \widehat{\E}(X_j^2 )  - \E( X_j^2 ) \right| > \epsilon_1\right)  \leq 2 \cdot \exp \left\{ - \frac{  n \epsilon_1^2 }{ 2 C_{x}^4 \log^{4} \eta } \right\}.
		\end{equation}
		
		Hence, using the union bound, we have 
		\begin{eqnarray*}
			& & P\left( \max_{j \in V} \left| \widehat{\E}( X_j^2)  - \E( X_j^2) \right| > \epsilon_1 \right) \leq 2 p \cdot \exp \left\{ - \frac{  n \epsilon_1^2 }{  2 C_{x}^4  \log^{4} \eta } \right\}. 
		\end{eqnarray*}
	\end{proof}
	
	\item[(ii)] $P(\zeta_3^c ) \leq 2 p. d \cdot \exp \left(- \frac{n}{ \kappa_1(n,p)^2 } \right) + 2 p \cdot \exp \left\{ - \frac{  n \epsilon_1^2 }{ D_{\max} \log^4 \eta } \right\}$ for some constants $D_{\max} > 0 $. 
	\begin{proof}
		We restate the condition in the set $\zeta_3$ as 
		$$
		\Big|  \frac{1}{n} \sum_{i =1}^{n} f \big( \widehat{\E}(X_j \mid X_{S_j}^{(i)} )  \big) - \E\big( f \big( \E( X_j \mid X_{S_j} ) \big) \big) \Big| < \epsilon_1. 
		$$
		
		In order to apply Hoeffding's inequality, we first show the bound for $\widetilde{E}(X_j \mid X_{S_j})$. Recall that $[\theta^*]_{S^c}$ and $[\widetilde{\theta}]_{S^c} = (0,0,...,0) \in \mathbb{R}^{|S^c|}$ by the definition of $S$, and $|S| \leq d$. In Appendix~\ref{SubSecProofLem12}, we showed that $\|\widetilde{\theta}_{ S } - \theta_{ S }^* \|_2 \leq \frac{5}{ \rho_{\min} } \sqrt{ d } \lambda_j$ for $\lambda_j \leq \frac{ \rho_{\min}^2 }{ 10 C_{x}^2 \rho_{\max} d  \log^2 \eta }$ with a high probability. Therefore, given Assumption~\ref{A3Con}, for all $i \in \{1,2,...,n\}$, 
		\begin{eqnarray*}
			\exp( \langle \widehat{\theta}_{S_j}, X_{S_j}^{(i)} \rangle ) 
			&=& \exp( \langle \widehat{\theta}_{S_j} - \theta_{S_j}^*, X_{S_j}^{(i)} \rangle ) \cdot \exp( \langle \theta_{S_j}^*, X_{S_j}^{(i)} \rangle ) \\
			&\leq& \exp( \| \widehat{\theta}_{S} - \theta_{S}^*\|_2 \| X_{S}^{(i)} \|_2 \rangle ) \cdot \exp( \langle \theta_{S}^*, X_{S}^{(i)} \rangle )\\
			& \leq  & \exp\left\{ \frac{5 C_{x} d~\lambda_j }{\rho_{\min}} \|X_{S}^{(i)} \|_{\infty} \right\}  \cdot \exp( \langle \theta_{S}^*, X_{S}^{(i)} \rangle )  \\
			& \leq  & \exp\left\{ \frac{5 C_{x} d~\lambda_j }{\rho_{\min}} \log(\eta) \right\}  \cdot C_{x} \log \eta \\
			& \leq  & \exp\left\{ \frac{ \rho_{\min} }{ 2 C_{x} \rho_{\max}  \log \eta } \right\}  \cdot C_{x} \log \eta.
		\end{eqnarray*}
		
		Therefore, 
		\begin{align*}
		f \big( \widehat{\E}(X_j^{(i)} \mid X_{S_j}^{(i)} )  \big) 
		\leq C_{x}^2 \cdot \exp\left\{ \frac{\rho_{\min} }{ C_{x} \rho_{\max} } \right\} \log^2 \eta + C_{x} \cdot \exp\left\{ \frac{\rho_{\min} }{ 2 C_{x} \rho_{\max} } \right\} \log \eta.
		\end{align*}
		
		Hence there exists a positive constant $D_1 > 0$ such that for all $i \in \{1,2,...,n\}$, 
		$$
		f \big( \widehat{\E}(X_j \mid X_{S_j}^{(i)} )  \big) \leq D_1 \log^2 \eta.
		$$
		
		Applying Hoeffding's inequality,, for any $\epsilon_1>0$ and any $j \in V$, 
		\begin{equation}
		P\left( \left| \widehat{\E}( f(\widetilde{E}(X_j \mid X_{S_j} ) ) )  - \E( f(\widetilde{E}(X_j \mid X_{S_j}) ) ) \right| > \epsilon_1 \right)  \leq 2 \cdot \exp \left\{ - \frac{  2 n \epsilon_1^2 }{ D_1^2 \log^4 \eta} \right\}.
		\end{equation}
		
		Hence, there exist some constants $D_{\max} > 0 $ such that
		\begin{align*}
		P\left( \max_{j \in V} \zeta_3^c \right) &\leq 1 -2 p. d \cdot \exp \left(- \frac{n}{ \kappa_1(n,p)^2 } \right) - 2 p \cdot \exp \left\{ - \frac{  n \epsilon_1^2 }{ D_{\max} \log^4 \eta } \right\}.
		\end{align*} 
%		
%		As we showed in Appendix~\ref{SecThmMainProof1}, if we set $\kappa_1(n,p) = C_{\max} d \log^4 \eta $ where $C_{\max} = \frac{8 \cdot  10^3 C_{x}^4 \cdot(2-\alpha)^2 }{  \alpha^2} \frac{ \rho_{\max} }{ \rho_{\min}^2 } $,
%		\begin{align*}
%		P\left( \max_{j \in V} \zeta_3^c \right) \leq 1 -2 p. d \cdot &\exp \left(- \frac{n}{ C_{\max}^2 d^2 \log^8 \eta } \right)  - 2 p \cdot \exp \left\{ - \frac{  n \epsilon_1^2 }{ D_{\max} \log^4 \eta } \right\}.
%		\end{align*}
%		
	\end{proof}
	
	\item[(iii)] $P(\zeta_4^c)  = 0$.
	
	\begin{proof}
		
		We restate the condition in the set $\zeta_4$ as 
		$$
		\Big|  \E\left( f( \E(X_j \mid X_{S_j}) ) - f( \widehat{\E}(X_j \mid X_{S_j}) ) \right)  \Big| < \epsilon_1. 
		$$
		
		%%%%%%%%%%%%%
		
		By the mean-value theorem, for some $v \in [0, 1]$,
		\begin{align*}
		&f \big( \widehat{\E}(X_j \mid X_{S_j} )  \big) - f \big( \E( X_j \mid X_{S_j} ) \big) \\
		&\quad  = f'\big(v \widehat{\E}(X_j \mid X_{S_j} ) + (1-v) \E(X_j \mid X_{S_j} ) \big) (\widehat{\E}(X_j \mid X_{S_j} ) -   \E( X_j \mid X_{S_j} ))\\
		&\quad = 2\big(v \widehat{\E}(X_j \mid X_{S_j} ) + (1-v) \E(X_j \mid X_{S_j} ) + 1/2 \big)
		(\widehat{\E}(X_j \mid X_{S_j} ) -   \E( X_j \mid X_{S_j} )).
		\end{align*}
		
		Therefore, 
		\begin{align*}
		&\E( f \big( \widehat{\E}(X_j \mid X_{S_j} )  \big) - f \big( \E( X_j \mid X_{S_j} ) \big) ) \\
		&\quad  = f'\big(v \widehat{\E}(X_j \mid X_{S_j} ) + (1-v) \E(X_j \mid X_{S_j} ) \big) (\widehat{\E}(X_j \mid X_{S_j} ) -   \E( X_j \mid X_{S_j} ))\\
		&\quad = 2\big(v \widehat{\E}(X_j \mid X_{S_j} ) + (1-v) \E(X_j \mid X_{S_j} ) +1/2 \big)
		(\widehat{\E}(X_j \mid X_{S_j} ) -   \E( X_j \mid X_{S_j} ))\\
		&\quad \leq \max \big| 2\big(v \widehat{\E}(X_j \mid X_{S_j} ) + (1-v) \E(X_j \mid X_{S_j} ) +1/2 \big) \big| \cdot  \E\left( \widehat{\E}(X_j \mid X_{S_j} ) -   \E( X_j \mid X_{S_j} ) \right) \\ &\quad = 0
		\end{align*}
		
		In the same manner, $\E( f \big( \E( X_j \mid X_{S_j} ) \big) - f \big( \widehat{\E}(X_j \mid X_{S_j} )  \big) ) \leq 0$. This completes the proof. 
		
	\end{proof}

\end{itemize}

\end{document}